\documentclass[10pt]{article} % For LaTeX2e

\usepackage[preprint]{rlj} % Should be uncommented for preprint versions

\usepackage{amssymb}            % Defines common symbols like \mathbb R
\usepackage{mathtools}          % Extends amsmath, providing common math tools
\usepackage{mathrsfs}           % Enables \mathscr, which can work in cases that \mathcal does not
\usepackage{graphicx}           % For including images
\usepackage{subcaption}         % Allows for the use of subfigures and subcaptions
\usepackage[space]{grffile}     % For spaces in image names
\usepackage{url}                % For displaying URLs
\usepackage{lipsum}             % For placeholder text

\usepackage{algorithm}
\usepackage{algorithmicx}
\usepackage{algpseudocode}
\usepackage{amsthm}
\newtheorem{theorem}{Theorem}
\newtheorem{lemma}{Lemma}
\newtheorem{assumption}{Assumption}
\usepackage{booktabs}
\usepackage{adjustbox}
\usepackage{placeins}

\newcommand{\edit}[1]{\textcolor{black}{#1}}
\newcommand{\hyp}{x}
\newcommand{\Hyp}{\mathcal{X}}
\newcommand{\noise}{\xi}

\title{Efficient Heteroscedastic Bayesian Optimization for Risk-Aware AutoRL}

\setrunningtitle{Efficient Heteroscedastic Bayesian Optimization for Risk-Aware AutoRL}

\author{Mingxuan Che\textsuperscript{1,$*$},
Tsung-Yuan Tseng\textsuperscript{3,4,$*$}, 
Theresa Eimer\textsuperscript{1,2}, \\
Marius Lindauer\textsuperscript{1,2},
Alexander von Rohr\textsuperscript{5}
}

\emails{\{m.che,t.eimer,m.lindauer\}@ai.uni-hannover.de, \\ tsung.yuan.tseng@cariad.technology, alexander.von.rohr@utn.de}

\affiliations{
$^{1}$\textbf{Leibniz University Hannover} \quad
$^{2}$\textbf{L3S Research Center} \quad 
$^{3}$\textbf{CARIAD SE}\\
$^{4}$\textbf{University of Wuppertal} \quad
$^{5}$\textbf{University of Technology Nuremberg}
\par % If including additional comments like below, use \par to add some whitespace. 
$^*$ Equal contribution; order decided by a fair coin flip
}

\contribution{
    We propose efficient and risk-averse heteroscedastic Bayesian optimization (ERAHBO), an adaptive re-sampling formulation of RAHBO \citep{makarova2021risk} that improves practical sample efficiency in HPO for RL. We provide sublinear regret guarantees along with an empirical comparison to common Bayesian optimization (BO) baselines.
    }
    {
    To integrate risk into the optimization process, we need to repeatedly evaluate each hyperparameter configuration. RAHBO sets a fixed number of evaluations for each configuration, which makes it computationally expensive on a runtime-intense domain like reinforcement learning.
    }

\contribution{
    We provide a dataset of RL learning outcomes mapping hyperparameters to post-training policy returns for 50 runs, enabling reproducible benchmarking of risk-aware hyperparameter optimization methods.
    }
    {
    Recent HPO performance datasets such as those provided in HPO-RL-Bench \citep{shala-automl24a} and ARLBench \citep{becktepe2024arlbench} include repeated evaluations for each configuration, but the number of seeds (up to ten) per hyperparameter configuration can be insufficient for a reliable estimation of heteroscedastic variance across the hyperparameter space. Our dataset addresses this by providing 50 seeds per configuration.
    }

 % Your keywords

\summary{
Reinforcement learning (RL) learning outcomes can be highly stochastic, and both expected performance and variability often depend strongly on hyperparameter (HP) configurations. This makes HP optimization (HPO) brittle: outcome variability translates into a practical risk of unsuccessful training after HPO. 

We propose risk-aware HPO through a mean-variance objective, dubbed efficient and risk-averse heteroscedastic Bayesian optimization (ERAHBO). It models both mean and variance, and improves sample efficiency via adaptive re-sampling rather than a fixed replication budget per HP configuration.
We show that our adaptive strategy improves efficiency and HPO performance across a wide range of RL environments.
We additionally provide an offline dataset of RL learning outcomes mapping HP configurations to post-training policy returns across 50 repeated runs per HP configuration to evaluate risk-aware HPO methods.
}

\begin{document}

\makeCover  % Create the cover page
\maketitle  % Make the title section

\begin{abstract}
Reinforcement learning (RL) has shown remarkable success across a wide range of complex tasks. However, RL outcomes can be highly stochastic, and both expected performance and variability often depend on hyperparameter (HP) configurations. We propose efficient and risk-averse heteroscedastic Bayesian Optimization (ERAHBO), a Bayesian optimization method that models both the mean and variance of learning outcomes as functions of the HP configurations. ERAHBO aims to identify HP configurations that achieve high average return while reducing variability across training runs, and it improves the sample efficiency of the HP optimization via adaptive re-sampling rather than a fixed budget per HP. Empirical evaluations across diverse RL algorithms and environments demonstrate that ERAHBO generally outperforms both risk-neutral and risk-averse baselines, delivering improved sample efficiency for risk-averse returns.
\end{abstract}

%%%%%%%%%%%%%%%%%%%%%%%%%%%%%%%%%%%%%%%%%%%%%%%%%%%%%%%%%%%%%%%%
%% Section: Submission of papers to RLJ/RLC
%%%%%%%%%%%%%%%%%%%%%%%%%%%%%%%%%%%%%%%%%%%%%%%%%%%%%%%%%%%%%%%%
\section{Introduction}

The performance of reinforcement learning (RL) algorithms is often highly dependent on the choice of hyperparameters (HPs) \citep{berkenkamp2016safe, islam2017reproducibility,henderson-aaai18a,zhang-aistats21a,eimer-icml23a,obandoceron-rlj24}. Moreover, in most RL experiments, inherent randomness arises from e.g., stochastic task behavior, exploratory choices, or hardware nondeterminism~\citep{zhuang2022randomness}. 
Recent work shows that even very good HP configurations may still lead to a high degree of randomness in learning outcomes~\citep{dierkes-ewrl25}.

Automated RL (AutoRL) aims to make RL work reliably out-of-the-box by automating, for example, the setting of HPs.
However, the stochasticity of RL training makes automated hyperparameter optimization brittle: the configuration that appears best during HPO can fail to reproduce its performance in subsequent runs, and high-variance HP configurations can yield poor policies with nontrivial probability, requiring costly retraining. In practice, we would like HPO to return configurations that \emph{reliably} achieve strong performance across runs, supporting reproducible and cost-effective RL experimentation. Therefore, it is an important research question how we can efficiently extract HP configurations that consistently lead to good performance, i.e., high average return with low variability across repeated training runs.

Bayesian optimization (BO) is a sample-efficient, model-based approach to HPO that is well suited to expensive black-box objectives and is widely used as an automated alternative to manual tuning and search heuristics~\citep{bischl-dmkd23a}.
However, classical BO approaches typically optimize the mean of the objective function, in RL typically the expected final return.
This ignores training stability across runs for each hyperparameter configuration, i.e., its \emph{risk}.

A natural way to incorporate this notion into HPO is to use \emph{risk-aware} Bayesian optimization methods that explicitly model both the mean performance and the input-dependent variability of outcomes, and optimize a risk-sensitive objective.
In particular, risk-averse heteroscedastic Bayesian optimization (RAHBO; ~\citet{makarova2021risk}) provides a principled sequential BO approach that maintains separate surrogate models for the mean and variance of the return, and uses a risk-aware acquisition function.
RAHBO can therefore be used for risk-aware HP tuning.
However, in its original form, estimating the variance model relies on a \emph{static} number of evaluations (replications) per HP configuration, which can be inefficient in RL where each run is computationally costly.

We propose a novel adaptive replication strategy for RAHBO that determines when to acquire more samples of a configuration.
The key idea is to allocate runs where they matter most: configurations that are promising under a risk-aware objective but still ambiguous due to uncertainty, rather than spending a uniform replication budget everywhere (see Fig.~\ref{fig:erahbo}).
Our resulting method is an \textbf{efficient risk-averse heteroscedastic Bayesian optimization (ERAHBO)} algorithm, designed for the high degrees of variance and simultaneously high computational loads we see in RL.

To support the study of risk in AutoRL, we curate a new dataset of RL learning outcomes for different algorithms and environments which enables systematic analysis of heteroscedasticity in RL hyperparameters and benchmarking of risk-aware HPO methods.

\section{Problem formulation}

\begin{figure}[t]
  \begin{subfigure}[t]{1\linewidth}
    \centering
    \includegraphics{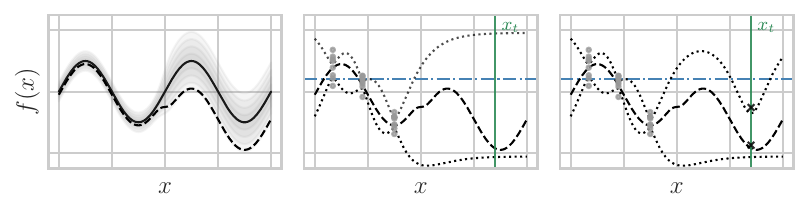}
    \label{sine}
  \end{subfigure}
  \caption{Adaptive re-sampling in ERHABO. \textbf{Left:} Example objective with two maxima at different noise levels, and the MV objective (dashed). \textbf{Middle:} Upper and lower confidence bounds for the MV objective (dotted) together with the incumbent LCB (blue dash-dot) after three configurations; the next query \(\hyp_t\) (chosen via UCB) is shown by the green vertical line. 
  \textbf{Right:} After two additional samples at \(\hyp_t\), the UCB falls below the incumbent LCB, so ERHABO does not need to re-sample.
  }
  \label{fig:erahbo}
\end{figure}

We consider a discounted Markov decision process (MDP). A (stochastic) policy $\pi(a\mid s)$ induces a distribution over trajectories $\tau \sim \pi(\tau)$. The performance of a fixed policy $\pi$ is its expected discounted return
$
J(\pi) := \mathbb{E}_{\edit{a} \sim \pi}\Big[\sum_{t=0}^{\infty}\gamma^t r(s_t,a_t)\Big]
$,
where $\gamma\in(0,1)$ is the discount factor and $r(s, a)$ is the reward function.

In HPO for RL, the decision variable is a hyperparameter configuration $\hyp \in \Hyp \subset \mathbb{R}^d$ of a learning algorithm $\mathcal{A}$, and the outcome is the trained policy produced by running $\mathcal{A}$ with $\hyp$. Because RL training is inherently stochastic (e.g., random initialization, exploration, and stochasticity encountered during training), running $\mathcal{A}$ with a fixed configuration $x$ induces a distribution over learned policies. We model the learning outcome as
$
\pi \sim \mathcal{A}(\pi; \hyp).
$
Crucially, the randomness we consider is the randomness of the \emph{learning outcome} $\pi$ (and thus of $J(\pi)$) induced by training.
We formulate risk-aware HPO as maximizing expected learning outcome while penalizing variability across learning outcomes:
\begin{align}
\label{eq:hpo}
\hyp^* \in \underset{\hyp \in \Hyp}{\mathrm{argmax}}\ 
\mathbb{E}_{\pi \sim \mathcal{A}(\pi; \hyp)} \big[J(\pi)\big]
- \alpha \mathbb{V}ar_{\pi \sim \mathcal{A}(\pi; \hyp)} \big[J(\pi)\big],
\end{align}
where $\alpha \ge 0$ defines the trade-off between mean and variance. For notational simplicity, define the mean performance
\begin{align}
\label{eq:f_def}
f(\hyp) := \mathbb{E}_{\pi \sim \mathcal{A}(\pi; \hyp)} \big[J(\pi)\big].
\end{align}

In Bayesian optimization (BO), we consider a sequential interaction problem with unknown objective $f:\Hyp\to\mathbb{R}$. In each BO round $t$, the algorithm chooses $\hyp_t \in \Hyp$, runs training once to obtain $\pi_t \sim \mathcal{A}(\pi; \hyp_t)$, and observes one realization of the learning outcome
\begin{align}
\label{eq:bo_observation}
y_t := J(\pi_t) = f(\hyp_t) + 
\noise(\hyp_t),
\end{align}
where $\noise(\hyp_t)$ is a zero-mean noise term capturing the stochasticity of the learning algorithm $\mathcal{A}$ when trained with hyperparameters $\hyp_t$ (i.e., variability across independent executions of the RL training with the same HPs $\hyp_t$), and is independent across rounds $t$. In the following, we drop the subscript $t$ when referring to a general $\hyp$ or when its meaning is clear from the context. Consequently,
\begin{align}
\mathbb{V}ar\!\left[\noise(\hyp)\right] = \mathbb{V}ar_{\pi \sim \mathcal{A}(\pi; \hyp)} \big[J(\pi)\big].
\end{align}

Following~\citet{makarova2021risk}, we assume that $f$ belongs to a reproducing kernel Hilbert space (RKHS) $\mathcal{H}_{\kappa}$ induced by a kernel $\kappa(\cdot,\cdot)$ satisfying $\kappa(\hyp,\hyp') \leq 1$ for all $\hyp,\hyp' \in \Hyp$, and that $\| f \|_{\kappa} \leq \mathcal{B}_f$ for some $\mathcal{B}_f>0$. We further assume $\noise(\hyp)$ is sub-Gaussian with a (possibly input-dependent) variance proxy $\rho^2(\hyp)$ uniformly bounded as $\rho(\hyp)\in[\underline{\rho},\overline{\rho}]$ for constants $\overline{\rho}\ge \underline{\rho}>0$.

Following \citet{sani2012risk}, \citet{makarova2021risk} and \citet{dai2023batch}, the mean-variance objective with variance proxy $\rho^2(\hyp)$ is
\begin{align}
\label{eq:mv_objective}
\mathrm{MV}(\hyp) = f(\hyp) - \alpha \rho^2(\hyp),
\end{align}
where a fixed $\alpha \geq 0$ expresses a user defined risk aversion.

We seek a sequence of evaluations $\{\hyp_t\}_{t=1}^{T}$ that attains high values of the risk-averse objective $\mathrm{MV}(\hyp)$ over $T$ BO rounds. Since each round $t$ collects $k_t$ independent samples (training runs) using $\hyp_t$, we evaluate performance via the sample-count cumulative regret
\begin{align}
\label{eq:cumulative_regret}
R_T := \sum_{t=1}^{T} k_t \Big(\mathrm{MV}(\hyp^*) - \mathrm{MV}(\hyp_t)\Big), \quad \text{where}  \quad  \hyp^* \in \arg\max_{\hyp \in \Hyp} \mathrm{MV}(\hyp).
\end{align}

\section{Related work}

This section gives a brief overview of BO for hyperparameter optimization (HPO), focusing on settings with stochastic and configuration-dependent outcomes, and then connects these ideas to HPO for RL, where training variability is often substantial.

Grid search and random search are simple baselines for HPO, but they become inefficient as the dimensionality of the HP space grows~\citep{bergstra-jmlr12a}. Bayesian optimization~\citep{garnett-book23a} improves sample efficiency for costly black-box objectives by fitting a surrogate model and selecting evaluations sequentially~\citep{bergstra-nips11a,snoek-nips12a-replace,springenberg-nips16a}. It is the backbone of many widely used HPO approaches~\citep{akiba-kdd19a,lindauer-jmlr22a,song2024vizier}.

The substantial run-to-run variability of RL training has been documented in prior work~\citep{henderson-aaai18a,obandoceron-rlj24} and motivates systematic and efficient approaches to HPO. HPO for RL is an emerging field (see \citet{parkerholder-jair22a} for an overview), but poses challenges due to high computational overhead and stochastic learning outcomes~\citep{agarwal2021deep,eimer-icml23a} arising from training instabilities~\citep{dierkes-ewrl25}. In this context, optimizing expected performance alone, often approximated via a small number of training runs~\citep{parkerholder-neurips20a,wan-automl22a}, may not capture reliability across runs. Our goal is to provide an objective and optimization strategy that better targets reliable HP configurations in RL.

Many classical BO methods assume homoscedastic noise and mismatched noise assumptions can degrade BO performance. Heteroscedastic Gaussian process models capture input-dependent noise explicitly~\citep{binois2018practical,binois2019replication}, but modeling variance alone does not prevent BO from selecting configurations with undesirable high variability in the learning outcomes.
To address this, risk-aware BO methods optimize objectives that explicitly trade off mean performance and variability. Prior work considers mean-variance objectives for heteroscedastic BO~\citep{makarova2021risk,dai2023batch}. In particular, RAHBO~\citep{makarova2021risk} uses a fixed evaluation budget per configuration to estimate variability, which is inefficient since it allocates significant resources to evaluate under-performing configurations. \citet{dai2023batch} addresses adaptive evaluation in a batch setting. In this paper, we evaluate RAHBO for HPO for RL problems and extend it with an adaptive evaluation strategy that speeds up the optimization process across many HPO tasks.
\section{Method}

This section describes our approach for risk-aware HPO under heteroscedastic training outcomes. We begin with the essential modeling and acquisition components of RAHBO~\citep{makarova2021risk}, and then present ERAHBO, an adaptive resampling strategy, together with a corresponding sublinear regret bound.

\subsection{RAHBO: risk-averse heteroscedastic Bayesian optimization}

We briefly summarize the components of risk-averse heteroscedastic Bayesian optimization (RAHBO)~\citep{makarova2021risk} that we build on.
The key idea of RAHBO is to model both the performance $f$ and variance proxy $\rho^2$ as separate Gaussian processes (GPs) to derive confidence bounds for a risk-averse variant of the Gaussian process upper confidence bound (GP-UCB) algorithm~\citep{srinivas-icml10a}.

RAHBO obtains mean and variance estimation to condition the GPs by repeatedly evaluating a configuration. In iteration $t$, it selects a configuration $\hyp_t$ and performs a fixed number $k$ of independent evaluations, yielding observations $\left\{ y_t^{i}\right\}_{i=1}^k$. The sample mean and variance are estimated as
\begin{equation}
    \hat{m}(\hyp_t) = \frac{1}{k} \sum_{i=1}^k y_t^{(i)}, 
    \quad 
    \hat{s}^2(\hyp_t) = \frac{1}{k-1} \sum_{i=1}^k( y_t^{(i)} - \hat{m}(\hyp_t))^2,
\end{equation}
which serve as noisy observations of $f(\hyp_t)$ and $\rho^2(\hyp_t)$ to condition the two GPs.

Let $\mathrm{UCB}_t^{f}(\hyp)$ denote an upper confidence bound for $f(\hyp)$, and $\mathrm{LCB}_t^{\mathrm{var}}(\hyp)$ a lower confidence bound for $\rho^2(\hyp)$, constructed from the two GP posteriors with appropriately chosen confidence parameters as in~\citet{makarova2021risk}. RAHBO selects the next query by maximizing an upper bound on $\mathrm{MV}(\hyp)$,
\begin{align}
\label{eq:rahbo_acq}
\hyp_{t} \in \arg\max_{\hyp\in\Hyp}\ \mathrm{UCB}_t^{f}(\hyp)\;-\;\alpha\,\mathrm{LCB}_t^{\mathrm{var}}(\hyp).
\end{align}

Under standard regularity assumptions, RAHBO achieves sublinear regret for the mean-variance objective with a fixed sample budget $k$ per iteration. For more details, see Appendix~\ref{apdx:rahbo}.

Our method retains the same modeling and risk-aware acquisition structure, but replaces the fixed budget $k$ with an adaptive re-sampling strategy that allocates evaluations based on estimated uncertainty, improving practical efficiency in high-cost settings such as RL.

\subsection{ERAHBO: efficient risk-averse heteroscedastic Bayesian optimization}

To resolve the limitation of RAHBO, we propose a confidence-based adaptive rule to automatically determine replications for \( k \in [k_{min}, k_{max}]\) in each BO iteration. 

\subsubsection{Confidence-based resampling}\label{adaptive rule}

While repeated sampling improves the estimation accuracy of both the mean and variance, it is not necessary to improve the estimate of the configuration once it can be certified as suboptimal. We construct the high-probability confidence bounds for the MV objective as
\begin{align}
\mathrm{UCB}_t^{\mathrm{MV}}(\hyp) &:= \mathrm{UCB}_t^f(\hyp) - \alpha\, \mathrm{LCB}_t^{\mathrm{var}}(\hyp), \quad
\mathrm{LCB}_t^{\mathrm{MV}}(\hyp) &:= \mathrm{LCB}_t^f(\hyp) - \alpha\, \mathrm{UCB}_t^{\mathrm{\mathrm{var}}}(\hyp).
\end{align}
The core idea of ERAHBO is to stop re-sampling a configuration once we know it is worse than the incumbent with high probability.

\paragraph{Incumbent lower bound}
We define the incumbent lower bound at iteration $t$ as
\begin{equation}\label{eq:inc_lcb}
B_t := \max_{\hyp \in \mathcal{D}_{t-1}} \mathrm{LCB}_t^{\mathrm{MV}}(\hyp),
\end{equation}
where $\mathcal{D}_{t-1}$ denotes the set of previously evaluated configurations. The quantity $B_t$ represents a high-confidence lower bound on the best MV value observed so far and serves as a conservative benchmark against which newly queried configurations are compared.

\paragraph{Stopping criterion}
At iteration $t$, the algorithm selects a configuration
\begin{equation}
\hyp_t \in \arg\max_{\hyp \in \Hyp} \mathrm{UCB}_t^{\mathrm{MV}}(\hyp),
\end{equation}
and begins collecting repeated evaluations at $\hyp_t$. Let $k$ denote the current number of samples collected at $\hyp_t$. As $k$ increases, the uncertainty of the mean-variance estimate decreases due to the effective observation noise scaling as $\rho^2(\hyp_t)/k$, leading to progressively tighter confidence bounds.

Sampling at $\hyp_t$ is terminated once the following condition is satisfied:
\begin{equation}
\label{eq:adaptive_rule}
\mathrm{UCB}_t^{\mathrm{MV}}(\hyp_t; k) \le B_t,
\end{equation}
or once a predefined maximum number of evaluations $k_{\max}$ is reached.
This stopping rule admits a clear interpretation. If even the optimistic estimate of the MV objective at $\hyp_t$ is no larger than the pessimistic estimate of the best competing configuration, then with high probability, $\hyp_t$ cannot outperform the incumbent and further replication is unnecessary. Conversely, as long as the confidence interval of $\mathrm{MV}(\hyp_t)$ overlaps with the incumbent lower bound, additional evaluations may be informative and are therefore allocated. Appendix~\ref{apdx:k_bound} shows that if a configuration is suboptimal relative to the incumbent lower bound, then ERAHBO will stop sampling it after finitely many samples. 
\subsubsection{Sublinear regret guarantee}\label{guarantee}

We show that the sublinear regret guarantees of RAHBO~\citep{makarova2021risk} carry over to a setting with varying replication counts.
Under the standard RKHS smoothness and sub-Gaussian noise assumptions used in RAHBO, ERAHBO preserves the sublinear regret guarantees of RAHBO when the number of replications per BO round is allowed to vary but remains uniformly bounded, \(k_t \in [k_{\min}, k_{\max}]\). In particular, with high probability, the sample-count cumulative regret satisfies
\begin{equation}
R_T =
\tilde O\!\left(
k_{\max}\sqrt{T\gamma_T} + \alpha k_{\max}\sqrt{T\Gamma_T}
\right),
\end{equation}
where \(\gamma_T\) and \(\Gamma_T\) are the information-gain terms for the mean and variance surrogate models, respectively. Thus, bounded adaptive replication does not change the no-regret behavior: whenever the corresponding RAHBO information-gain terms are sublinear, ERAHBO has vanishing average sample-count regret. Appendix~\ref{apdx:rahbo} gives the full theorem statement, assumptions, and proof.

\section{Experiments}\label{sec:experiment}

We validate our approach on three popular RL algorithms (DQN~\citep{mnih-nature15a}, SAC~\citep{haarnoja-icml18a}, and PPO~\citep{schulman-arxiv17a}) on Brax~\citep{freeman-neurips21a}, XLand-Minigrid~\citep{nikulin2023-neurips23} and Classic Control environments~\citep{gymnax2022github}. 

\noindent
$\textbf{BO baselines}.$ We compare our method against two baselines. GP-UCB~\citep{srinivas-icml10a} implemented by the work of~\citet{makarova2021risk} is our risk-neutral baseline.
RAHBO~\citep{makarova2021risk} is the state-of-the-art risk-averse BO algorithm in our problem setting.
For GP-UCB, we use $k=20$ observations per configuration; for RAHBO, we use two variants, $k=2$ and $k=20$; and ERAHBO starts with $k_{\min} = 2$ and resamples each configuration until~\eqref{eq:adaptive_rule} is satisfied or $k_{\max}=20$ is reached.
For RAHBO and ERAHBO, we use a heteroscedastic GP for modeling $f(x)$ and a homoscedastic GP for $\rho^2(x)$.
We follow common practices by setting $\beta_t = 2.5$. 
We preprocess the dataset using the outcome processing proposed in~\citet{song2024vizier} and choose $\alpha = 1$ for the processed data to have a meaningful mean–variance trade‑off. Generally, the MV objective is scale sensitive (multiplying the performance by a positive scalar will change the optimal configuration) which makes online outcome processing more difficult.
% if the processing is not linear.
For confidence-based resampling, we set $\beta_{\text{stop}}=1$ for~\eqref{eq:adaptive_rule} to allow optimistic stopping. Before the BO procedure, we determine the GP hyperparameters by maximizing the marginal likelihood over $5$ Sobol-sampled initial points that are the same for all experiments. We repeat each experiment $20$ times with new initial points for every repetition. To tackle the high-dimensional search space, we apply the GP kernel hyperparameter prior that scales with the dimensionality of the search space, as proposed in~\citet{hvarfner2024vanilla}. 

\textbf{RL Tasks.} We use ARLBench~\citep{becktepe2024arlbench} to produce performance datasets on which we run our experiments.
In total, we use 5 environments for DQN, 10 for PPO and 4 for SAC.
For each environment, we compute a dataset consisting of $512$ configurations with $50$ seeds each for more repeatable experimentation.\footnote{The code and the dataset can be found at~\href{https://github.com/LUH-AI/Efficient-Risk-Averse-BO}{\texttt{https://github.com/LUH-AI/Efficient-Risk-Averse-BO}}.}
The full algorithm-environment combination is documented in Appendix~\ref{app-sec:algo-env}.
We use the configuration spaces proposed in ARLBench and its default network architectures.
As the performance objective, we use the mean evaluation return across 128 evaluation episodes at the end of training. 
Since environments differ in reward scales, we report the normalized regret to the average score of BO's initial points in each environment's dataset. 
We report additional empirical results in Appendix~\ref{app:full_results}.

\subsection{Can ERAHBO find reliable configurations?}

\begin{figure}
    \centering
    \includegraphics[width=0.95\linewidth]{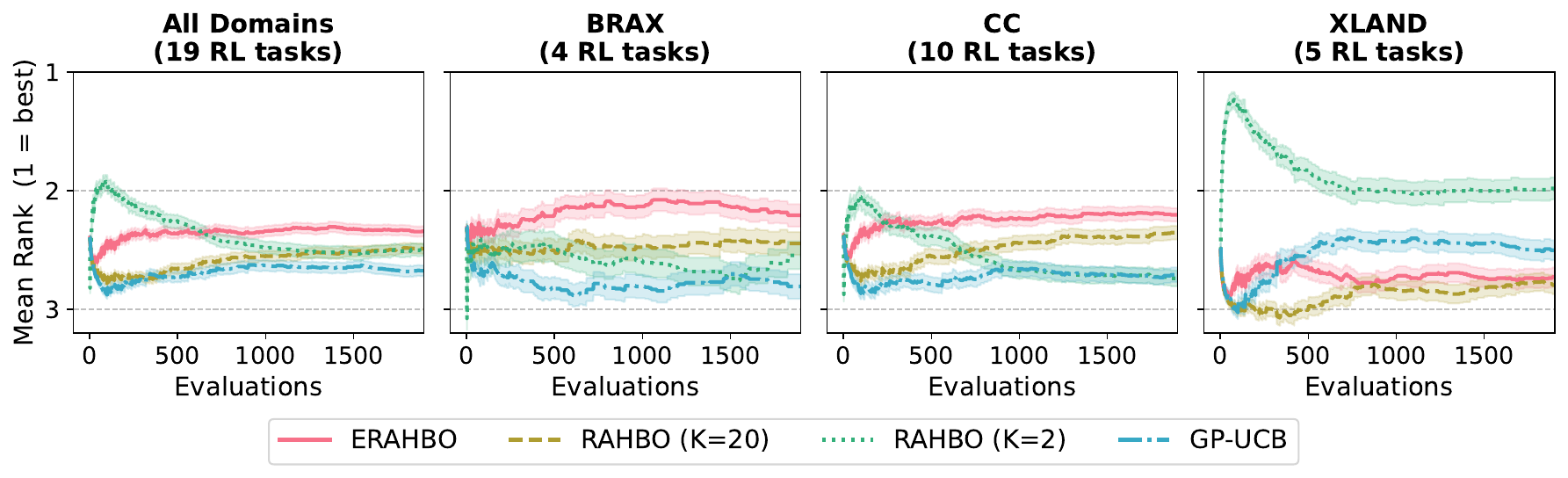}
    \caption{Mean-variance regret ranks aggregated by domain with lines and shaded area showing the mean and standard error, respectively. ERAHBO performs best overall. We observe different optimal values of $k$ in RAHBO across domains.}
    \label{fig:rank_over_time}
\end{figure}

We optimize the hyperparameters of PPO, DQN and SAC to see if ERAHBO can outperform RAHBO through its adaptive sampling mechanism.
Figure~\ref{fig:rank_over_time} shows that both RAHBO and ERAHBO outrank the GP-UCB baseline, but that the best value for RAHBO's $k$ varies across environments.
On Brax for example, $k=20$ outperforms $k=2$ while the opposite is true in XLand.
The environments ERAHBO does not perform well on are predominantly from XLand, which contains environments for which we see only very few good performances overall. This makes it hard to find a good incumbent and leads to many repetitions preventing exploration.

Overall, ERAHBO outperforms both RAHBO variants in a total of $7$ out of $19$ algorithm-environment pairs, and at least one RAHBO variation $16$ out of $19$ times.
In Table~\ref{tab:algo_rank_mean}, the Wilcoxon signed-rank test shows a statistically significant improvement of ERAHBO over RAHBO with $k=20$ and GP-UCB.
This means the adaptive sampling can automatically strike a balance between the different versions of RAHBO and eliminates the need for manually selecting $k$. We document all rank metrics-related results in Appendix~\ref{app-sec:rank_regret}, including a comparison of ERAHBO with RAHBO using different repetitions $k$.

\begin{table}[htbp]
  \centering
  \caption{Mean $\pm$ standard error of algorithm rank across experiments (lower is better). Bold values are not statistically significantly worse than best.}
  \begin{adjustbox}{max width=\linewidth}
  \begin{tabular}{lcccc}
    \toprule
    Metric & ERAHBO & RAHBO~(k=20) & RAHBO~(k=2) & GP-UCB \\
    \midrule
    Simple Regret MV & \textbf{1.79} $\pm$ \textbf{0.77} & 2.42 $\pm$ 1.04 & \textbf{2.37 }$\pm$ \textbf{1.09} & 3.42 $\pm$ 0.88 \\
    Cumulative Regret MV & \textbf{1.79} $\pm$ \textbf{0.77} & 2.58 $\pm$ 0.82 & \textbf{2.32} $\pm$ \textbf{1.34} & 3.32 $\pm$ 0.86 \\
    \bottomrule
  \end{tabular}
  \end{adjustbox}
  \label{tab:algo_rank_mean}
\end{table}

\subsection{Is ERAHBO making HPO more sample efficient?}
ERAHBO not only balances additional evaluations for the best performance tradeoff, but it is also more efficient in terms of evaluations than RAHBO or GP-UCB. In Table~\ref{tab:threshold_crossing_agg}, we report the evaluation number where the algorithms achieve a regret threshold measured as the percent of the initial design's regret, and we count the tasks where all algorithms reach this threshold for easier comparability. 
ERAHBO achieves sample efficiency comparable to RAHBO with $k=2$, reaching the regret thresholds after a similar number of evaluations. However, as discussed above, RAHBO with $k=2$ yields worse configurations after the full evaluation budget. Compared with RAHBO with $k=20$ and GP-UCB, ERAHBO reaches almost all the regret thresholds significantly faster.

\begin{table}[htbp]
  \centering
  \caption{
  Mean $\pm$ standard error number of evaluations needed to improve over regret thresholds. We average tasks where all methods reach the threshold (count per method in brackets).}
  \begin{adjustbox}{max width=\linewidth}
  \begin{tabular}{lrrrr}
    \toprule
    Threshold & ERAHBO & RAHBO (k=20) & RAHBO (k=2) & GP-UCB  \\
    \midrule
    75\% ($n_{\mathrm{common}}=15$) & $40.8 \pm 5.9$ (18) & $62.3 \pm 9.1$ (17) & $80.6 \pm 43.2$ (17) & $83.7 \pm 25.1$ (17) \\
    50\% ($n_{\mathrm{common}}=12$) & $107.6 \pm 14.9$ (16) & $132.7 \pm 15.7$ (16) & $139.2 \pm 91.6$ (15) & $197.7 \pm 52.5$ (14) \\
    25\% ($n_{\mathrm{common}}=11$) & $193.6 \pm 32.7$ (13) & $304.6 \pm 41.5$ (13) & $147.9 \pm 72.2$ (12) & $359.2 \pm 98.5$ (12) \\
    \bottomrule
  \end{tabular}
  \end{adjustbox}
  \label{tab:threshold_crossing_agg}
\end{table}

\begin{figure}
    \centering
    \includegraphics[width=0.95\linewidth]{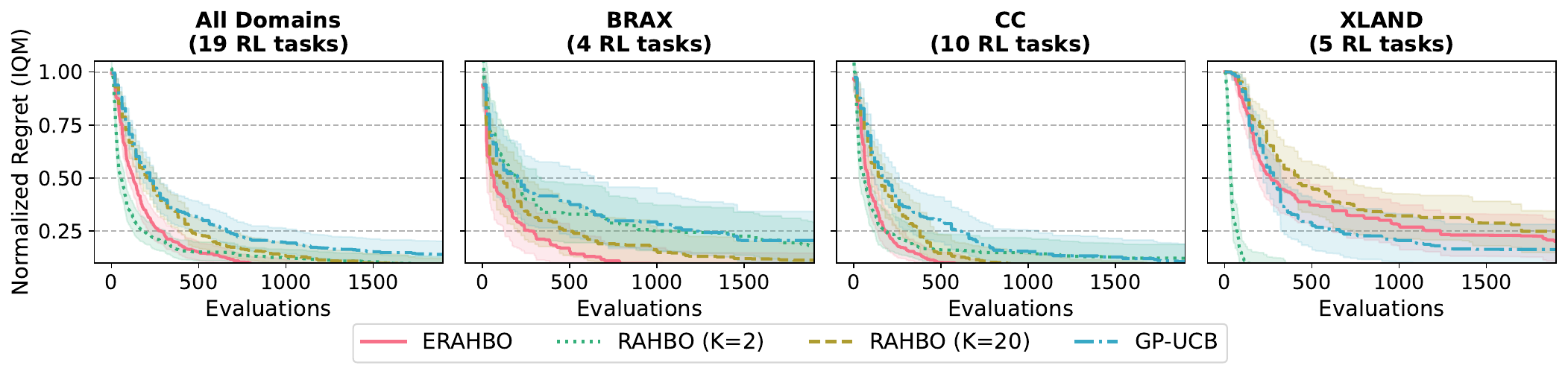}    
    \caption{Domain-aggregated normalized regret over evaluation budgets. Lines and shaded areas show the inter-quantile mean (IQM) and $95\%$ bootstrapped confidence interval (CI), respectively.}
    \label{fig:regret_aggregated}
\end{figure}

Furthermore, Figure~\ref{fig:regret_aggregated} shows normalized mean regret for all domains. 
Compared to GP-UCB, both ERAHBO and RAHBO can have better anytime performance. 
Therefore, they are better suited to low-budget HPO, as is often necessary in RL.
As in the overall ranks, we observe $k$ being an essential HP for RAHBO. 
With a lower $k$, RAHBO is very efficient for small budgets, but cannot effectively utilize larger optimization budgets.
With $k=20$, on the other hand, RAHBO becomes inefficient with small budgets.
ERAHBO can adapt to different budgets without HPO for its own HPs, making it a more flexible method. We note that the pattern depends on the risk-sensitivity parameter $\alpha$, since large values can prevent BO from exploring high-mean performance configurations. More detailed results are available in Appendix~\ref{app-sec:regret_over_time} and~\ref{app-sec:threshold_efficiency}.

\subsection{Does the adaptive sampling focus on well-performing configurations?}
\label{sec-k-adaptation}
Beyond being more efficient than RAHBO, we want to verify that ERAHBO actually prioritizes promising configurations and spends very few evaluations on poorly performing ones.
Figure~\ref{fig:ks} shows the distribution of repetitions across all evaluations.
These results indicate that ERAHBO allocates more budget to better configurations.
Additionally, we note that scheduling $\beta$ in the incumbent lower bound in \eqref{eq:inc_lcb} further improves the performance of ERAHBO (see Appendix~\ref{app-sec:schedule_ERAHBO_beta}).
\begin{figure}
    \centering
    \includegraphics[width=0.95\linewidth]{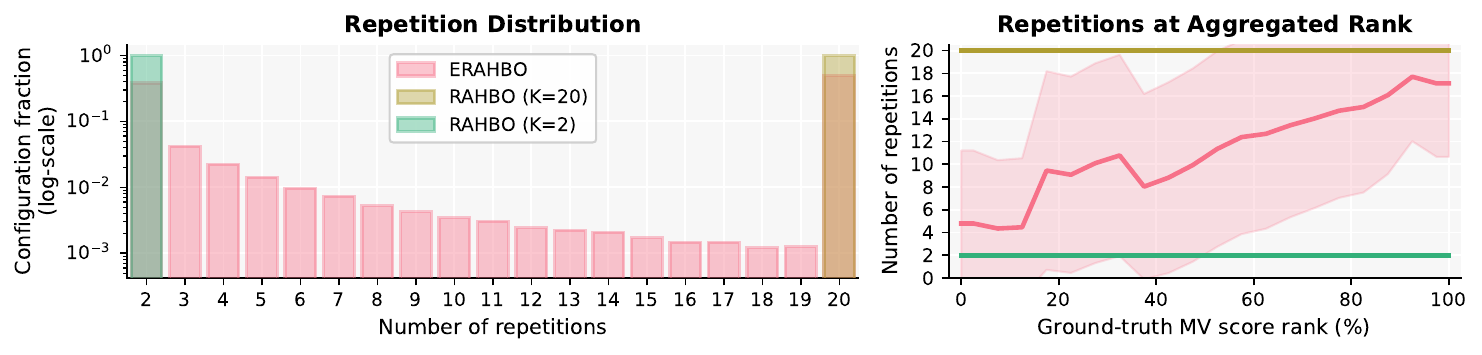}
    \caption{Repetition numbers in dataset experiments. \textbf{Left}: Compared with RAHBO using fixed sampling repetition, ERAHBO rejects the majority of configurations without spending the full budget. \textbf{Right}: ERAHBO invests more repetition in promising configurations with higher MV scores, demonstrating adaptivity through the confidence-based resampling behavior. 
    }
    \label{fig:ks}
\end{figure}

\section{Conclusion}
We identify variability between repeated runs as a key difficulty factor for HPO in RL and propose to use risk-aware optimization to find an ideal variability-performance tradeoff.
To do so efficiently for computationally expensive RL algorithms, we introduce ERAHBO, an adaptive sampling mechanism that automatically adds more samples only in regions of the search space that show high uncertainty.
We provide sublinear regret guarantees and demonstrate empirically that ERAHBO improves upon its baselines in terms of overall performance and efficiency.
Furthermore, the ERAHBO successfully identifies which configurations to prioritize and which to discard early.

We show that ERAHBO can improve the practical efficiency of risk-aware hyperparameter optimization by adaptively allocating samples to configurations where uncertainty matters most. However, our approach optimizes a mean-variance objective, which is only one notion of risk: the risk parameter $\alpha$ is scale-sensitive and does not directly target tail events or failure probabilities. Moreover, our regret guarantee is a conservative worst-case bound over all sample-count schedules $k_t\in[k_{\min},k_{\max}]$ rather than a tight characterization of the specific adaptive rule. Finally, our empirical results rely on offline datasets of learning outcomes; while this enables broad and reproducible comparisons, it is not identical to fully online HPO with fresh training runs.

While we demonstrate the benefits of risk-aware HPO and ERAHBO's adaptive budget allocation for RL algorithms, our empirical evaluation shows that not all environments benefit equally, as we see in the XLand environments.
The dataset we release with this paper will facilitate further research into the factors that influence risk in different RL settings.
This will help the community and us to further develop adaptive risk-aware HPO methods for the wide range of RL task settings.
We see ERAHBO's sampling strategy as a first step in this direction. 

%%%%%%%%%%%%%%%%%%%%%%%%%%%%%%%%%%%%%%%%%%%%%%%%%%%%%%%%%%%%%%%%
%% Appendices
%%%%%%%%%%%%%%%%%%%%%%%%%%%%%%%%%%%%%%%%%%%%%%%%%%%%%%%%%%%%%%%%
\FloatBarrier
\clearpage
\appendix

\section{Regret bound for ERAHBO}\label{apdx:rahbo}

We follow the setup of~\citet{makarova2021risk}. Throughout, let \(k_t:=k(\hyp_t)\in[k_{\min},k_{\max}]\) with \(k_{\min}\ge 2\), so that the unbiased sample variance is well-defined. We assume normalized kernels, \(\kappa(\hyp,\hyp)\le 1\) and \(\kappa^{\mathrm{var}}(\hyp,\hyp)\le 1\) for all \(\hyp\in\Hyp\). We consider the BO observation model (Eq.~\eqref{eq:bo_observation}) with heteroscedastic noise, and denote by \(\mu^f_{t-1}(\hyp\mid\Sigma)\) and \(\sigma^f_{t-1}(\hyp\mid\Sigma)\) the GP posterior mean and standard deviation of the mean model with noise matrix \(\Sigma\).

\setcounter{lemma}{0}
\begin{lemma}[Lemma 7 in~\citet{kirschner2018information}]\label{lemma1}
Let \(f\in\mathcal H_\kappa\), and let \(\mu_t(\cdot)\) and \(\sigma_t^2(\cdot)\) be the GP posterior mean and variance \(\lambda>0\). Assume that the observations \((\hyp_t,y_t)_{t\ge 1}\) satisfy Eq.~\eqref{eq:bo_observation}. Then for all \(t\ge 1\) and \(\hyp\in\Hyp\), with probability at least \(1-\delta\),
\begin{align}\label{eq:confidence_bound}
|\mu_{t-1}(\hyp)-f(\hyp)|
\le
\underbrace{\left(
\sqrt{2\ln\!\left(
\frac{{\det(\lambda\Sigma_t+K_t)}^{1/2}}
{\delta\,\det(\lambda\Sigma_t)^{1/2}}
\right)}
+\sqrt{\lambda}\|f\|_\kappa
\right)}_{:=\beta_t}
\,\sigma_{t-1}(\hyp).
\end{align}
\end{lemma}

In each BO round \(t\), ERAHBO collects \(k_t\) independent samples at \(\hyp_t\), denoted \(\{y_t^{(i)}\}_{i=1}^{k_t}\), and forms \(\hat m(\hyp_t):=k_t^{-1}\sum_{i=1}^{k_t} y_t^{(i)}\). Using Lemma~\ref{lemma1} for the mean GP posterior, define
\begin{align}
\label{eq:ucb_lcb_f}
\mathrm{UCB}^f_t(\hyp)
&:=\mu^f_{t-1}(\hyp\mid\widehat\Sigma_{t-1})
+\beta_t^f\,\sigma^f_{t-1}(\hyp\mid\widehat\Sigma_{t-1}),\\
\mathrm{LCB}^f_t(\hyp)
&:=\mu^f_{t-1}(\hyp\mid\widehat\Sigma_{t-1})
-\beta_t^f\,\sigma^f_{t-1}(\hyp\mid\widehat\Sigma_{t-1}).
\end{align}

We make the same smoothness assumption on the variance proxy as in~\citet{makarova2021risk}.

\begin{assumption}[Assumption 1 in \cite{makarova2021risk}]
\label{assumption1}
    The variance-proxy \( \rho^2(\hyp) \) belongs to an RKHS induced by some kernel \( \kappa^{\text{var}} \), i.e., \( \rho^2 \in \mathcal{H}_{\kappa^{\text{var}}} \), and its RKHS norm is bounded \( \| \rho^2 \|_{\kappa^{\text{var}}} \leq B_{\text{var}} \) for some finite \( B_{\text{var}} > 0 \). Moreover, the noise \( \noise(\hyp) \) in Eq. (\ref{eq:bo_observation}) is strictly \( \rho(\hyp) \)-sub-Gaussian, i.e., \( \text{Var}[\noise(\hyp)] = \rho^2(\hyp) \) for every \( \hyp \in \Hyp \).
\end{assumption}

\begin{assumption}[Analogue of Assumption 2 in~\citet{makarova2021risk}]
\label{assumption2}
Let \(\mathcal F_{t-1}\) denote the BO history before round \(t\). The noise \( \eta(\hyp_t,k_t) \) in \(\hat{s}^2(\hyp_t) = \rho^2(\hyp_t) + \eta(\hyp_t,k_t)\) is conditionally \( \rho_\eta(\hyp_t,k_t) \)-sub-Gaussian given \(\mathcal F_{t-1}\), \(\hyp_t\), and \(k_t\), with known \( \rho_\eta^2(\hyp_t,k_t) \). The realizations \( \{ \eta(\hyp_t,k_t) \}_{t \ge 1} \) are conditionally independent across \( t \) given the corresponding histories and selected pairs \((\hyp_t,k_t)\).
\end{assumption}

In the variable-sample setting, the variance-observation proxy may depend on both the query point and the replication count, i.e., $\rho_\eta^2(\hyp_t,k_t)$. The regret analysis only requires a uniform bound on these proxies. Thus, it suffices to assume that there exists $\mathcal R^2$ such that
\[
\mathcal R^2 \ge \sup_{\hyp\in\Hyp,\; k\in[k_{\min},k_{\max}]} \rho_\eta^2(\hyp,k),
\qquad
\mathcal R^2 = \frac{2\bar{\rho}^4}{k_{\min}-1}
\ \text{is valid under strict sub-Gaussianity}.
\]

Let \( \mu^{\mathrm{var}}_{t-1} \) and \( \sigma^{\mathrm{var}}_{t-1} \) denote the variance-GP posterior mean and standard deviation built with \(\Sigma^{\mathrm{var}}_{t-1}:=\mathrm{diag}(\rho_\eta^2(\hyp_1,k_1),\ldots,\rho_\eta^2(\hyp_{t-1},k_{t-1}))\). Using Lemma~\ref{lemma1} for the variance GP posterior, define
\begin{align}
\label{eq:ucb_lcb_var}
\mathrm{UCB}^{\mathrm{var}}_t(\hyp)
&:=\mu^{\mathrm{var}}_{t-1}(\hyp)
+\beta_t^{\mathrm{var}}\,\sigma^{\mathrm{var}}_{t-1}(\hyp),\\
\mathrm{LCB}^{\mathrm{var}}_t(\hyp)
&:=\mu^{\mathrm{var}}_{t-1}(\hyp)
-\beta_t^{\mathrm{var}}\,\sigma^{\mathrm{var}}_{t-1}(\hyp).
\end{align}

Following~\citet{makarova2021risk}, we construct
\begin{align}
\label{eq:Sigma_hat_appendix}
\widehat\Sigma_{t-1}
:=\mathrm{diag}_{s<t}\!\left(
\frac{\min\{\mathrm{UCB}^{\mathrm{var}}_{t-1}(\hyp_s),\overline\rho^2\}}{k_s}
\right).
\end{align}
On the simultaneous variance-confidence event, \( \rho^2(\hyp_s)\le \mathrm{UCB}^{\mathrm{var}}_{t-1}(\hyp_s)\) for all \(s\le t-1\), and by assumption \(\rho^2(\hyp_s)\le\overline\rho^2\). Hence the construction above is used as a conservative noise matrix for the sample-mean observations. Equivalently, throughout the proof we condition on the event that
\[
\frac{\rho^2(\hyp_s)}{k_s}
\le
\frac{\min\{\mathrm{UCB}^{\mathrm{var}}_{t-1}(\hyp_s),\overline\rho^2\}}{k_s},
\qquad s<t,
\]
so Lemma~\ref{lemma1} applies to the mean GP with noise matrix \(\widehat\Sigma_{t-1}\).

For convenience, define \(N_T:=\sum_{t=1}^T k_t\) and \(\bar\beta_T^\ell:=\max_{1\le s\le T}\beta_s^\ell\), \(\ell\in\{f,\mathrm{var}\}\).

We use the sequence-based information gain terms
\begin{align}
\label{eq:gamma_bar_def}
\overline{\gamma}_T
&:=
\max_{\substack{\hyp_1,\ldots,\hyp_T\in\Hyp\\ k_1,\ldots,k_T\in[k_{\min},k_{\max}]}}
\sum_{t=1}^{T}
\frac{1}{2}\ln\!\left(
1+\frac{\bigl(\sigma^f_{t-1}\bigr)^2(\hyp_t\mid D_{t-1})}{\overline\rho^2/k_t}
\right),\\
\label{eq:Gamma_def}
\Gamma_T
&:=
\max_{\substack{\hyp_1,\ldots,\hyp_T\in\Hyp\\ k_1,\ldots,k_T\in[k_{\min},k_{\max}]}}
\sum_{t=1}^{T}
\frac{1}{2}\ln\!\left(
1+\frac{\bigl(\sigma^{\mathrm{var}}_{t-1}\bigr)^2(\hyp_t)}{\rho_{\eta,t}^2}
\right).
\end{align}
Here \(D_{t-1}:=\mathrm{diag}(\overline\rho^2/k_1,\ldots,\overline\rho^2/k_{t-1})\).

We now adapt the regret analysis of~\citet{makarova2021risk} to variable sample counts.

\setcounter{theorem}{0}
\begin{theorem}\label{thm2_apdx}
Consider any \( f \in \mathcal{H}_\kappa \) with \( \| f \|_\kappa \le B_f \), and the BO observation model in Eq.~\eqref{eq:bo_observation} with an unknown variance proxy \(\rho^2(\hyp)\) that satisfies Assumptions~\ref{assumption1} and~\ref{assumption2}. 
Let \( \{ \hyp_t \}_{t=1}^T \) denote the sequence of query points selected by ERAHBO over \(T\) BO rounds, where in round \(t\) the algorithm collects \(k_t \in [k_{\min},k_{\max}]\) samples at \(\hyp_t\). 
We assume that the aggregated observations used to update the mean and variance GPs satisfy the observation models stated above; in particular, this holds when \(k_t\) is fixed before observing the samples in round \(t\).
Let \( \{ \beta_t^f \}_{t=1}^T \) and \( \{ \beta_t^{\mathrm{var}} \}_{t=1}^T \) be defined according to Lemma~\ref{lemma1} (with \(\lambda=1\)). 
Set \(\rho(\cdot)\in[\underline\rho,\overline\rho]\) and let \(\mathcal R^2\) be any uniform bound on the variance-GP observation proxies.
Then, with probability at least \(1-2\delta\), for all \(T \ge 1\),
\begin{align}
\label{eq:thm2_bound_apdx}
R_T \;\le\;
2\,\bar\beta^f_T \, k_{\max}\,
\sqrt{\frac{2T\, \overline{\gamma}_T}{\ln\!\bigl(1 + k_{\max} / \overline{\rho}^2\bigr)}}
\;+\;
2\,\alpha \,\bar\beta_T^{\mathrm{var}} \, k_{\max}\,
\sqrt{\frac{2T\, \Gamma_T}{\ln\!\bigl(1 + \mathcal{R}^{-2}\bigr)}},
\end{align}
where \(\bar\beta_T^f := \max_{1\le s\le T}\beta_s^f\), \(\bar\beta_T^{\mathrm{var}} := \max_{1\le s\le T}\beta_s^{\mathrm{var}}\), and \(\overline{\gamma}_T\) and \(\Gamma_T\) are defined in \eqref{eq:gamma_bar_def}--\eqref{eq:Gamma_def}.
\end{theorem}

\begin{proof}
The proof follows the same steps as Appendix A.4 in~\citet{makarova2021risk}, with the modifications needed for variable sample counts.

\textbf{Step 1 (Confidence bounds for \(\mathrm{MV}\)).}
By the same argument as in~\citet{makarova2021risk}, with probability at least \(1-2\delta\),
\[
\mathrm{LCB}_t^{\mathrm{MV}}(\hyp)\le \mathrm{MV}(\hyp)\le \mathrm{UCB}_t^{\mathrm{MV}}(\hyp),
\qquad \forall \hyp\in\Hyp,\ t\ge 1,
\]
where
\begin{align}
\mathrm{UCB}_t^{\mathrm{MV}}(\hyp)
&=\mathrm{UCB}_t^f(\hyp)-\alpha\,\mathrm{LCB}_t^{\mathrm{var}}(\hyp),\\
\mathrm{LCB}_t^{\mathrm{MV}}(\hyp)
&=\mathrm{LCB}_t^f(\hyp)-\alpha\,\mathrm{UCB}_t^{\mathrm{var}}(\hyp).
\end{align}

\textbf{Step 2 (Instantaneous and cumulative regret).}
Let \(r_t:=\mathrm{MV}(\hyp^*)-\mathrm{MV}(\hyp_t)\). As in~\citet{makarova2021risk},
\begin{align}
\label{eq:inst_regret_bound}
r_t
&\le
\mathrm{UCB}_t^{\mathrm{MV}}(\hyp_t)-\mathrm{LCB}_t^{\mathrm{MV}}(\hyp_t)
\nonumber\\
&=2\beta_t^f\,\sigma^f_{t-1}(\hyp_t\mid\widehat\Sigma_{t-1})
+2\alpha\beta_t^{\mathrm{var}}\,\sigma^{\mathrm{var}}_{t-1}(\hyp_t).
\end{align}
Hence
\begin{align}
\label{eq:cum_regret_split}
R_T
&=\sum_{t=1}^T k_t r_t \nonumber\\
&\le
2\bar\beta_T^f \sum_{t=1}^T k_t\,\sigma^f_{t-1}(\hyp_t\mid\widehat\Sigma_{t-1})
+2\alpha\bar\beta_T^{\mathrm{var}} \sum_{t=1}^T k_t\,\sigma^{\mathrm{var}}_{t-1}(\hyp_t).
\end{align}

\textbf{Step 3 (Bounding the mean-model term).}
Let \(v_t:=\overline\rho^2/k_t\) and \(u_t:=\bigl(\sigma^f_{t-1}\bigr)^2(\hyp_t\mid D_{t-1})/v_t\).
Since
\[
\frac{\min\{\mathrm{UCB}^{\mathrm{var}}_{t-1}(\hyp_s),\overline\rho^2\}}{k_s}
\le \frac{\overline\rho^2}{k_s}=v_s
\qquad \forall s<t,
\]
we have \(\widehat\Sigma_{t-1}\preceq D_{t-1}\), and therefore \(\sigma^f_{t-1}(\hyp_t\mid\widehat\Sigma_{t-1})\le\sigma^f_{t-1}(\hyp_t\mid D_{t-1})\).
By Cauchy--Schwarz,
\begin{align}
\sum_{t=1}^{T} k_t\,\sigma^f_{t-1}(\hyp_t\mid\widehat\Sigma_{t-1})
&\le
\sum_{t=1}^{T} k_t\,\sigma^f_{t-1}(\hyp_t\mid D_{t-1})
\nonumber\\
&=
\sum_{t=1}^{T} k_t \sqrt{v_t}\sqrt{u_t}
=
\overline\rho\sum_{t=1}^{T}\sqrt{k_t}\sqrt{u_t}
\nonumber\\
&\le \overline\rho\sqrt{N_T}\sqrt{\sum_{t=1}^{T}u_t}.
\label{eq:weighted_cs_mean}
\end{align}
Moreover, since \(\lambda=1\) and \(\kappa(\hyp,\hyp)\le 1\), \(0\le u_t\le k_t/\overline\rho^2\le k_{\max}/\overline\rho^2\). Thus,
\[
u_t
\le
\frac{k_{\max}/\overline\rho^2}{\ln\!\bigl(1+k_{\max}/\overline\rho^2\bigr)}
\ln(1+u_t),
\]
and summing over \(t\) gives
\begin{align}
\sum_{t=1}^{T}u_t
&\le
\frac{2k_{\max}/\overline\rho^2}{\ln\!\bigl(1+k_{\max}/\overline\rho^2\bigr)}
\sum_{t=1}^{T}\frac{1}{2}\ln(1+u_t)
\nonumber\\
&\le
\frac{2k_{\max}\,\overline\gamma_T/\overline\rho^2}
{\ln\!\bigl(1+k_{\max}/\overline\rho^2\bigr)}.
\end{align}
Substituting into~\eqref{eq:weighted_cs_mean}, we obtain
\begin{align}
\sum_{t=1}^{T} k_t\,\sigma^f_{t-1}(\hyp_t\mid\widehat\Sigma_{t-1})
\le
\sqrt{
\frac{2k_{\max}N_T\,\overline\gamma_T}
{\ln\!\bigl(1+k_{\max}/\overline\rho^2\bigr)}
}.
\label{eq:sum_sigma_mean}
\end{align}

\textbf{Step 4 (Bounding the variance-model term).}
Define \(w_t:=\bigl(\sigma^{\mathrm{var}}_{t-1}\bigr)^2(\hyp_t)/\rho_{\eta,t}^2\). By Cauchy--Schwarz,
\begin{align}
\sum_{t=1}^{T} k_t\,\sigma^{\mathrm{var}}_{t-1}(\hyp_t)
&\le
\sqrt{\sum_{t=1}^{T} k_t}\,
\sqrt{\sum_{t=1}^{T} k_t\bigl(\sigma^{\mathrm{var}}_{t-1}\bigr)^2(\hyp_t)}
\nonumber\\
&\le
\sqrt{k_{\max}N_T}\,
\sqrt{\sum_{t=1}^{T}\bigl(\sigma^{\mathrm{var}}_{t-1}\bigr)^2(\hyp_t)}.
\label{eq:weighted_cs_var}
\end{align}
Since \(\lambda=1\) and \(\kappa^{\mathrm{var}}(\hyp,\hyp)\le 1\), \(0\le w_t\le \rho_{\eta,t}^{-2}\). Hence
\[
\rho_{\eta,t}^2 w_t
\le
\frac{1}{\ln\!\bigl(1+\rho_{\eta,t}^{-2}\bigr)}\ln(1+w_t)
\le
\frac{1}{\ln\!\bigl(1+\mathcal R^{-2}\bigr)}\ln(1+w_t),
\]
where the second inequality uses \(\rho_{\eta,t}^2\le \mathcal R^2\). Summing over \(t\) yields
\begin{align}
\sum_{t=1}^{T}\bigl(\sigma^{\mathrm{var}}_{t-1}\bigr)^2(\hyp_t)
&=\sum_{t=1}^{T}\rho_{\eta,t}^2 w_t
\nonumber\\
&\le
\frac{1}{\ln\!\bigl(1+\mathcal R^{-2}\bigr)}
\sum_{t=1}^{T}\ln(1+w_t)
\nonumber\\
&=\frac{2\Gamma_T}{\ln\!\bigl(1+\mathcal R^{-2}\bigr)}.
\end{align}
Substituting into~\eqref{eq:weighted_cs_var}, we obtain
\begin{align}
\sum_{t=1}^{T} k_t\,\sigma^{\mathrm{var}}_{t-1}(\hyp_t)
\le
\sqrt{
\frac{2k_{\max}N_T\,\Gamma_T}
{\ln\!\bigl(1+\mathcal R^{-2}\bigr)}
}.
\label{eq:sum_sigma_var}
\end{align}

\textbf{Step 5 (Combine bounds).}
Substituting~\eqref{eq:sum_sigma_mean} and~\eqref{eq:sum_sigma_var} into~\eqref{eq:cum_regret_split} yields~\eqref{eq:thm2_bound_apdx} with \(N_T\le Tk_{\max}\).
\end{proof}

\paragraph{Remark on adaptive resampling.}
The bound is stated for bounded sample counts \(k_t \in [k_{\min},k_{\max}]\) under the aggregate observation model used by RAHBO. ERAHBO's practical stopping rule chooses \(k_t\) by repeatedly checking the MV-UCB at the current point. Since the confidence bounds are uniform over points and rounds, and since the inner loop contains only finitely many checks \(k\in[k_{\min},k_{\max}]\), this provides a valid high-confidence stopping certificate whenever the intermediate GP updates satisfy the stated concentration assumptions, after applying a union or anytime bound over the inner-loop checks if needed. A fully optional-stopping analysis of the stopped sample mean and sample variance is beyond the scope of this work; the theorem should therefore be read as showing that bounded variable replication itself does not degrade the RAHBO regret rate, while our stopping rule is the practical mechanism used to choose such bounded replications.

\clearpage
\subsubsection*{Acknowledgments}
\label{sec:ack}
Mingxuan Che acknowledges funding by the European Union (ERC, ``ixAutoML'', grant no.101041029). Theresa Eimer and Marius Lindauer acknowledge funding by the German Research Foundation (DFG) under LI 2801/7-1 and LI 2801/10-1.
Alexander von Rohr is funded by the Deutsche Forschungsgemeinschaft (DFG) under project 468806714 of the Emmy Noether Programme. Alexander von Rohr also gratefully acknowledges funding from the European Union (ERC, ConSequentIAL, 101165883). Views and opinions expressed are however those of the author(s) only and do not necessarily reflect those of the European Union or the European Research Council. Neither the European Union nor the granting authority can be held responsible for them.
%%%%%%%%%%%%%%%%%%%%%%%%%%%%%%%%%%%%%%%%%%%%%%%%%%%%%%%%%%%%%%%%
%% NOTE: THIS MARKS THE END OF THE "MAIN TEXT"
%%%%%%%%%%%%%%%%%%%%%%%%%%%%%%%%%%%%%%%%%%%%%%%%%%%%%%%%%%%%%%%%

%%%%%%%%%%%%%%%%%%%%%%%%%%%%%%%%%%%%%%%%%%%%%%%%%%%%%%%%%%%%%%%%
%% Bibliography
%%%%%%%%%%%%%%%%%%%%%%%%%%%%%%%%%%%%%%%%%%%%%%%%%%%%%%%%%%%%%%%%
% local edit
\bibliography{lib/local_strings,lib/local,lib/local_lib,lib/local_proc}

\bibliographystyle{rlj}

%%%%%%%%%%%%%%%%%%%%%%%%%%%%%%%%%%%%%%%%%%%%%%%%%%%%%%%%%%%%%%%%
% AUTHOR: If your paper has no supplementary materials, you may 
%         comment out the line below, which creates the title for
%         the supplementary materials.
%%%%%%%%%%%%%%%%%%%%%%%%%%%%%%%%%%%%%%%%%%%%%%%%%%%%%%%%%%%%%%%%
\beginSupplementaryMaterials

\section{Finite-sample certificates}\label{apdx:k_bound}

This section provides a local certificate for ERAHBO's stopping rule in the case where the currently queried configuration can be shown to be suboptimal relative to the incumbent lower bound. Concretely, we derive a sufficient condition under which, after collecting enough samples at the current configuration $\hyp_t$, its optimistic mean-variance estimate drops below the incumbent lower bound with the current confidence bounds, so that additional sampling at $\hyp_t$ is unnecessary.

Fix a BO round $t$ and define
\[
B_t := \max_{\hyp \in \mathcal{D}_{t-1}} \mathrm{LCB}_t^{\mathrm{MV}}(\hyp),
\]
computed once at the start of round $t$ and kept fixed during the inner sampling loop at $\hyp_t$.
For $k\ge 2$, let $\mu^{\mathrm{MV}}_{t,k}(\hyp_t)$, $\sigma^f_{t,k}(\hyp_t)$, and $\sigma^{\mathrm{var}}_{t,k}(\hyp_t)$ denote the corresponding posterior quantities after incorporating the first $k$ samples collected at $\hyp_t$, and define
\[
m_t(k) := B_t - \mu^{\mathrm{MV}}_{t,k}(\hyp_t).
\]
Also define
\[
\mathrm{UCB}_{t,k}^{\mathrm{MV}}(\hyp_t)
:=
\mu^{\mathrm{MV}}_{t,k}(\hyp_t)
+ \beta_t^f \sigma^f_{t,k}(\hyp_t)
+ \alpha \beta_t^{\mathrm{var}} \sigma^{\mathrm{var}}_{t,k}(\hyp_t).
\]

\begin{lemma}[Finite-sample certificate of suboptimality]\label{lem:finite_sample_cert}
Assume $m_t(k)>0$ and
\[
\sigma^f_{t,k}(\hyp_t) \le \frac{\rho(\hyp_t)}{\sqrt{k}},
\qquad
\sigma^{\mathrm{var}}_{t,k}(\hyp_t) \le \frac{c_{\mathrm{var}}}{\sqrt{k}},
\]
for some constants $\rho(\hyp_t)>0$ and $c_{\mathrm{var}}>0$. If
\begin{equation}\label{eq:k_suff}
k \ge
\frac{\bigl(\beta_t^f \rho(\hyp_t) + \alpha \beta_t^{\mathrm{var}} c_{\mathrm{var}}\bigr)^2}{m_t(k)^2},
\end{equation}
then
\[
\mathrm{UCB}_{t,k}^{\mathrm{MV}}(\hyp_t) \le B_t.
\]
\end{lemma}

\begin{proof}
By definition of $\mathrm{UCB}_{t,k}^{\mathrm{MV}}(\hyp_t)$, it suffices that
\[
\beta_t^f \sigma^f_{t,k}(\hyp_t)
+ \alpha \beta_t^{\mathrm{var}} \sigma^{\mathrm{var}}_{t,k}(\hyp_t)
\le m_t(k).
\]
Using the assumed bounds on $\sigma^f_{t,k}(\hyp_t)$ and $\sigma^{\mathrm{var}}_{t,k}(\hyp_t)$, this is implied by
\[
\frac{\beta_t^f \rho(\hyp_t) + \alpha \beta_t^{\mathrm{var}} c_{\mathrm{var}}}{\sqrt{k}}
\le m_t(k),
\]
which is equivalent to \eqref{eq:k_suff}.
\end{proof}

\noindent\textbf{Takeaway.}
If the current posterior mean $\mu^{\mathrm{MV}}_{t,k}(\hyp_t)$ lies below the incumbent lower bound $B_t$ by a margin $m_t(k)>0$, then the stopping rule is triggered once the confidence radius becomes smaller than this margin. In particular, the sufficient sample size in \eqref{eq:k_suff} scales as $1/m_t(k)^2$: candidates that are closer to the incumbent require substantially more samples to certify suboptimality. The threshold also increases with the uncertainty multipliers $\beta_t^f$ and $\beta_t^{\mathrm{var}}$, with the local noise level $\rho^2(\hyp_t)$, and with the variance-model constant $c_{\mathrm{var}}$, formalizing that ERAHBO allocates more samples only when uncertainty could plausibly change the mean-variance decision.

\newpage

\section{Additional empirical results}
\label{app:full_results}
This section contains additional empirical results for ERAHBO and the full experiment details for the results presented in Section~\ref{sec:experiment}.

\subsection{Algorithm-environment combination}
\label{app-sec:algo-env}
Table~\ref{app-tab:algo_env_domain_summary} shows the RL algorithms and environments used in the experiment, and each pair of them corresponds to a dataset for risk-aware BO tasks. By `task', we refer to a dataset for the algorithm-environment pair. The dataset is generated using ARLbench~\citep{becktepe2024arlbench}, including all available RL tasks at the time of publication. Each dataset contains 512 configurations generated by a Sobol sequence, and each configuration includes policy returns across 50 random seeds. 

Some datasets contain configurations that return \texttt{nan}, a common issue in RL benchmarks. In our experiments, we filter out these~\texttt{nan} configurations, since failure handling is beyond the scope of this work and is left for future work. In detail, there are 18 infeasible configurations from \texttt{ppo\_cc\_pendulum} dataset, 48 from~\texttt{ppo\_cc\_continuous\_mountain\_car}, and 182 from~\texttt{ppo\_brax\_halfcheetah}.
In addition, we exclude the dataset~\texttt{dqn\_xland\_doorkey} because for this task, all configurations sampled by the Sobol sequence yield zero performance. Both datasets, before and after feasibility filtering, are available online at~\href{https://github.com/LUH-AI/Efficient-Risk-Averse-BO}{\texttt{https://github.com/LUH-AI/Efficient-Risk-Averse-BO}}.

\begin{table}[ht]
\centering
\caption{RL algorithms and environments from ARLBench used in the experiment by domain}
\label{app-tab:algo_env_domain_summary}
\footnotesize
\begin{tabular}{@{} l @{\hspace{2em}} p{4cm} @{\hspace{2em}} p{3cm} @{\hspace{2em}} p{3.5cm} @{}}
\toprule
\textbf{Algorithm} & \textbf{Classic Control (CC)} & \textbf{Brax} & \textbf{XLand} \\
\midrule
\textbf{DQN} & \texttt{acrobot}, \texttt{cartpole}, \texttt{mountain\_car} & --- & \texttt{empty\_random},\texttt{four\_rooms} \\
\addlinespace
\textbf{PPO} & \texttt{acrobot}, \texttt{cartpole}, \texttt{continuous\_mountain\_car}, \texttt{mountain\_car}, \texttt{pendulum} & \texttt{fast}, \texttt{halfcheetah} & \texttt{door\_key}, \texttt{empty\_random}, \texttt{four\_rooms} \\
\addlinespace
\textbf{SAC} & \texttt{continuous\_mountain\_car}, \texttt{pendulum} & \texttt{fast}, \texttt{halfcheetah} & --- \\
\bottomrule
\end{tabular}
\end{table}

\subsection{Ranking statistics}
\label{app-sec:rank_regret}
Tables~\ref{tab:algo_rank_mean} and \ref{tab:algo_rank_median} show the aggregated mean and median rank for final simple and cumulative regret of the mean-variance objective.
We see similar trends in performance over time: RAHBO and ERAHBO outperform GP-UCB, with RAHBO with $k=2$ ranking similarly to ERAHBO but with a larger inter-quantile range.

\begin{table}[htbp]
  \centering
  \caption{Median algorithm rank $[\text{25th percentile}, \text{75th percentile}]$ aggregated across experiments; lower is better.}
  \begin{adjustbox}{max width=\linewidth}
  \begin{tabular}{lcccc}
    \toprule
    Metric & ERAHBO & RAHBO ($k=20$) & RAHBO ($k=2$) & GP-UCB \\
    \midrule
Simple Regret MV & 2.00 [1.00, 2.00] & 3.00 [1.50, 3.00] & 2.00 [1.50, 3.00] & 4.00 [3.00, 4.00] \\
    Cumulative Regret MV & 2.00 [1.00, 2.00] & 3.00 [2.00, 3.00] & 2.00 [1.00, 4.00] & 4.00 [3.00, 4.00] \\
    \bottomrule
  \end{tabular}
  \end{adjustbox}
  \label{tab:algo_rank_median}
\end{table}

Figures~\ref{app-fig:perf_over_time_mean}
shows the mean ranks over time per environment, aggregated across RL algorithms. As in the main paper, we see clear differences across domains, especially in XLand, but ERAHBO ranks best overall. 

Figure~\ref{app-fig:mode-rank} shows the rank distribution. In this experiment, the rank distributions across algorithms are similar; however, the fraction colormap reveals that the RAHBO~$k=2$~rank distribution is bimodal, with a concentration at the best and worst ranks. 

In addition to the rank result with RAHBO repetitions $k = 2$ and $k=20$ shown in Figure~\ref{fig:rank_over_time}, we run RAHBO with $k\in\{5, 8, 11, 14, 17\}$ and present the domain-aggregated rank over time in~Figures~\ref{app-fig:diff-k}.  After around 750 evaluations, ERAHBO outperforms all the RAHBO repetition variants.

\begin{figure}
    \centering
    \includegraphics[width=\linewidth]{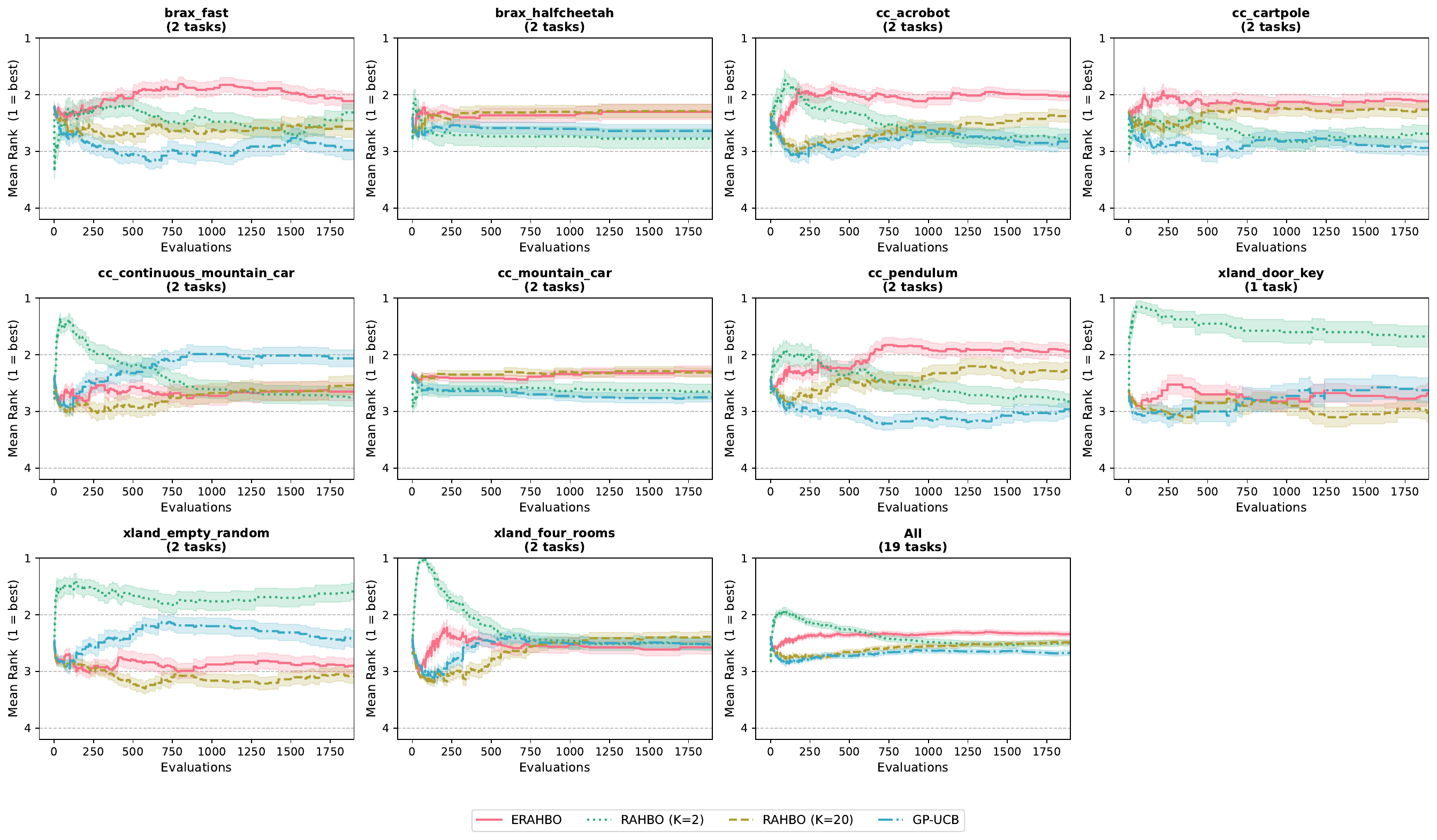}
    \caption{Mean algorithm rank aggregated by RL environments. The line and shaded area show the mean and $1\times$ standard error, respectively.
    }
    \label{app-fig:perf_over_time_mean}
\end{figure}

\begin{figure}
    \centering
    \includegraphics[width=\linewidth]{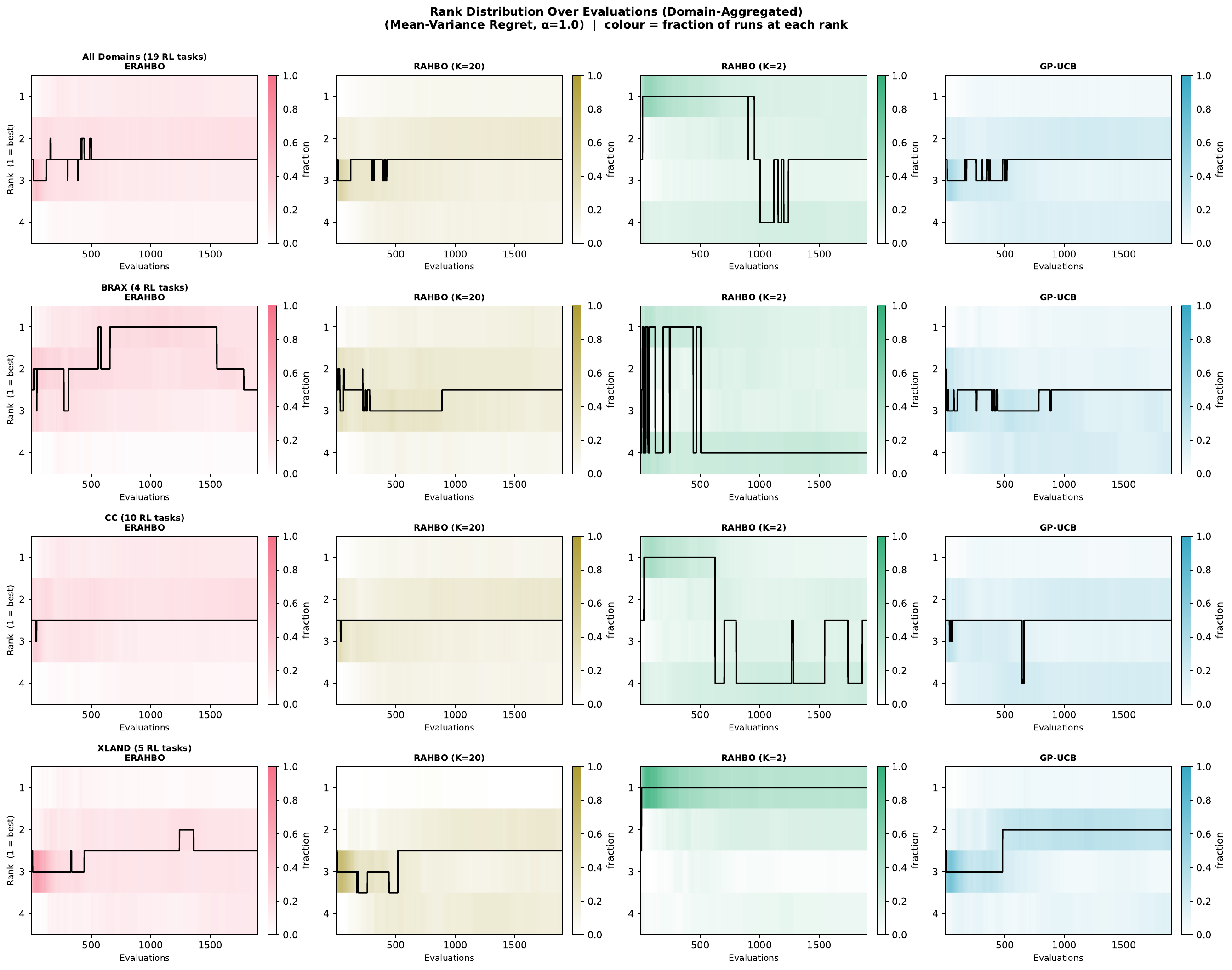}
    \caption{Rank distribution over time. Rows and columns correspond to the aggregation domain and method, respectively, as compared in the experiment (e.g., Figure~\ref{fig:rank_over_time}). The solid line and the colormap show the rank mode and the fraction of rank appearances. 
}
    \label{app-fig:mode-rank}
\end{figure}

\begin{figure}
    \centering
    \includegraphics[width=\linewidth]{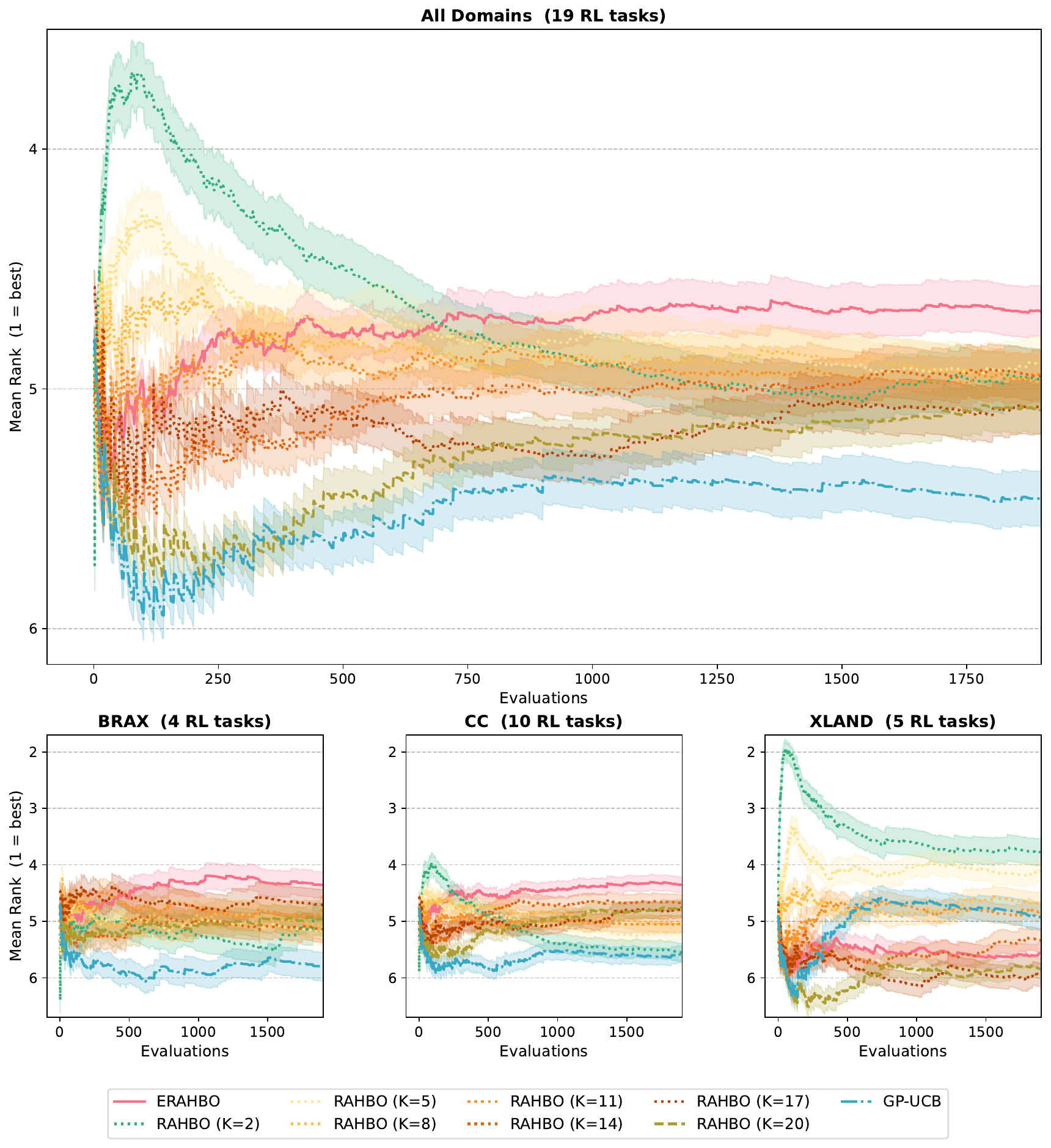}
    \caption{Mean algorithm rank aggregated by domain with RAHBO $k \in \{2, 5, 8, 11, 14, 17, 20\}$. The line and shaded area show the mean and $1\times$ standard error, respectively.
    }
    \label{app-fig:diff-k}
\end{figure}

\newpage
\subsection{Aggregated regret over time}
\label{app-sec:regret_over_time}
Figure~\ref{app-fig:checkpoint_normalized_mean_regret}-\ref{app-fig:checkpoint_xland} show the normalized mean-variance regret at different evaluation budgets $t$ for domains CC, Brax, and XLand. Figures~\ref{app-fig:regret_over_time_mean}
shows the mean
regret over time per environment and algorithm. 
Overall, the tasks are quite different: for some, we see large improvements over the initial design (e.g., on acrobot), whereas in others, not much can be gained (e.g., on mountain car). 
While the curves look slightly different for different algorithms, the overall trend seems to be environment-dependent.

\begin{figure}
    \centering
    \includegraphics[width=\linewidth]{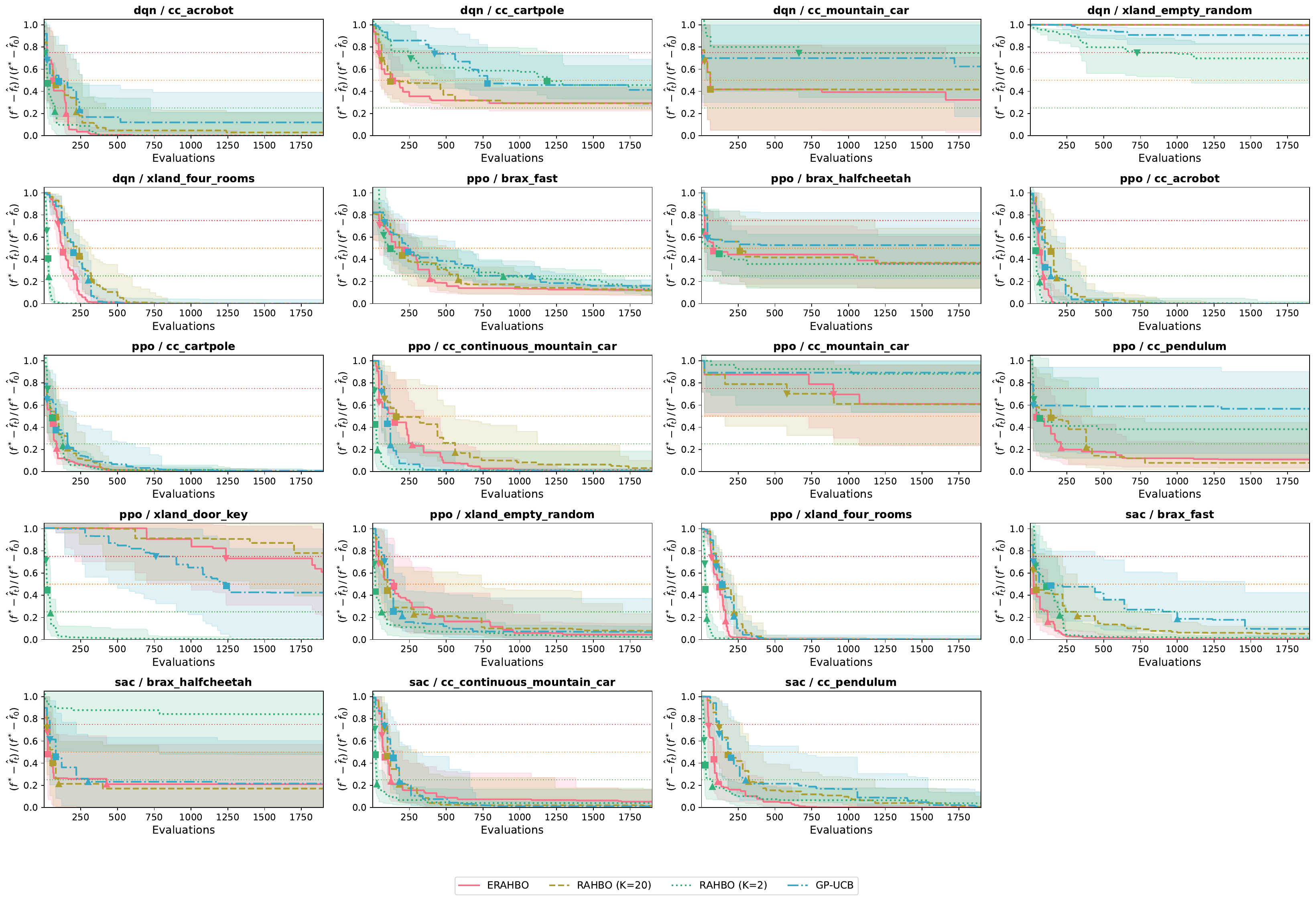}
    \caption{Algorithm normalized regret for individual RL task. The line and shaded area show the IQM and $95\%$ bootstrapped CI, respectively. The colored marker indicates the point at which the algorithm's normalized regret exceeds the threshold (\% of initial-design regret).}
    \label{app-fig:regret_over_time_mean}
\end{figure}

\subsection{Sample efficiency}
\label{app-sec:threshold_efficiency}

\begin{figure}
    \centering
    \includegraphics[width=\linewidth]{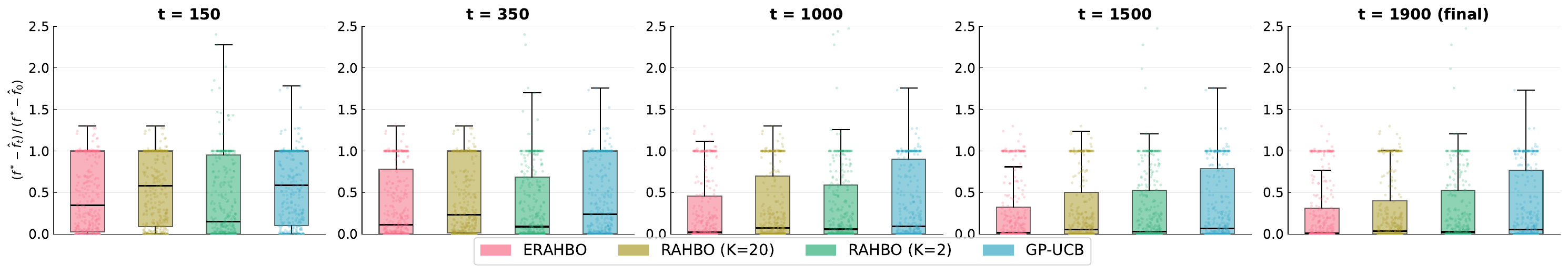}
    \caption{Normalized mean-variance regret at different evaluation budget $t$. Mean regret is normalized by the initial design's mean regret.}
    \label{app-fig:checkpoint_normalized_mean_regret}
\end{figure}

\begin{figure}
    \centering
    \includegraphics[width=\linewidth]{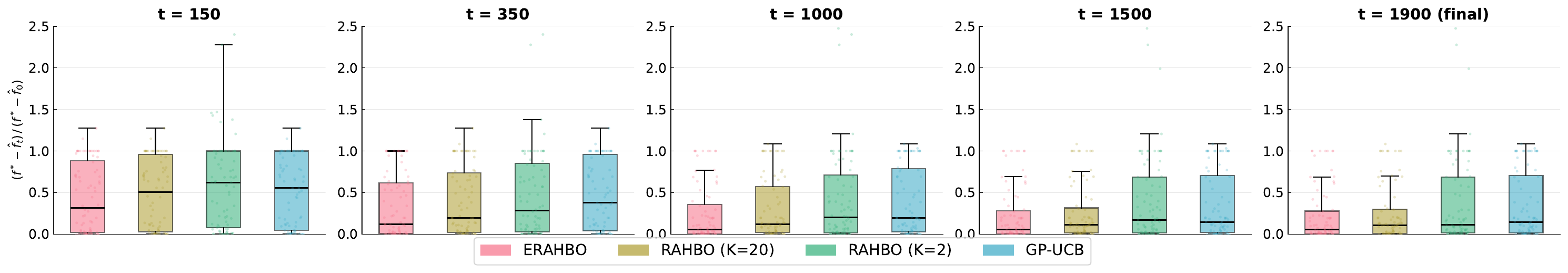}
    \caption{Normalized mean-variance regret at different evaluation budget $t$ in domain Brax. Mean regret is normalized by the initial design's mean regret.}
    \label{app-fig:checkpoint_brax}
\end{figure}

\begin{figure}
    \centering
    \includegraphics[width=\linewidth]{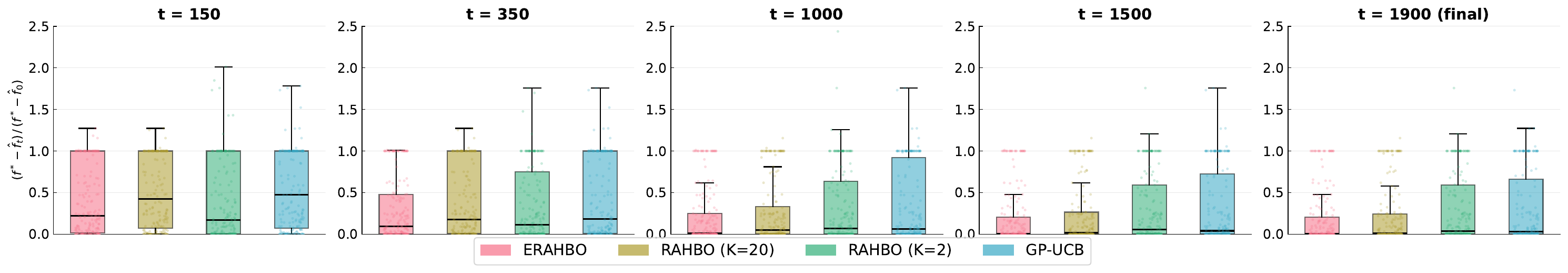}
    \caption{Normalized mean-variance regret at different evaluation budget $t$ in domain CC. Mean regret is normalized by the initial design's mean regret.}
    \label{app-fig:checkpoint_cc}
\end{figure}

\begin{figure}
    \centering
    \includegraphics[width=\linewidth]{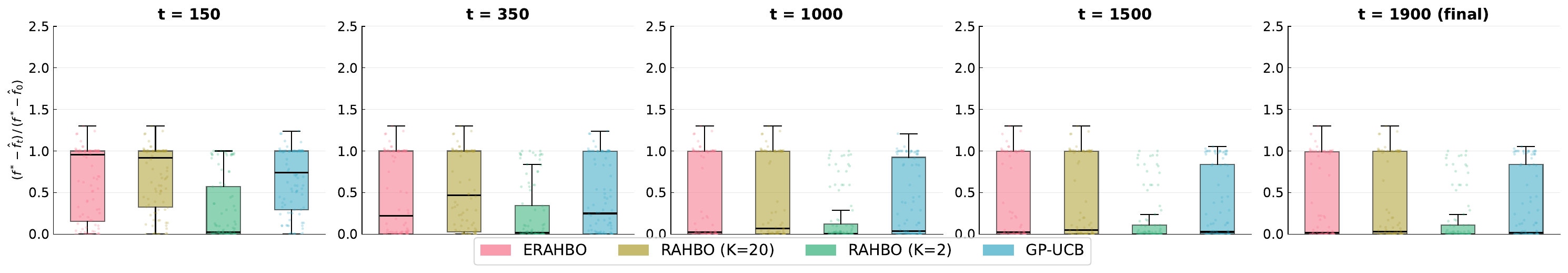}    \caption{Normalized mean-variance regret at different evaluation budget $t$ in domain Xland. Mean regret is normalized by the initial design's mean regret.}
    \label{app-fig:checkpoint_xland}
\end{figure}

Tables~\ref{tab:threshold_crossing_mean} %and \ref{tab:threshold_crossing_median}
shows how quickly ERAHBO, RAHBO and GP-UCB improve over the initial design. We clearly see that ERAHBO can be more efficient than RAHBO in most cases due to discarding poor configurations early on.
\begin{table}[htbp]
\centering
  \caption{First evaluation steps where the IQM of normalized MV regret drops below threshold (\% of initial-design regret). "---" means not reached.}
  \begin{adjustbox}{max width=\linewidth}
  \begin{tabular}{lrrrrrrrrrrrr}
    \toprule
    ~ & \multicolumn{3}{c}{ERAHBO} & \multicolumn{3}{c}{RAHBO (K=20)} & \multicolumn{3}{c}{RAHBO (K=2)} & \multicolumn{3}{c}{GP-UCB} \\
    \cmidrule(lr){2-4} \cmidrule(lr){5-7} \cmidrule(lr){8-10} \cmidrule(lr){11-13}
    Combo & $75\%$ & $50\%$ & $25\%$ & $75\%$ & $50\%$ & $25\%$ & $75\%$ & $50\%$ & $25\%$ & $75\%$ & $50\%$ & $25\%$ \\
    \midrule
    dqn/cc\_acrobot & $21$ & $65$ & $151$ & $21$ & $81$ & $221$ & $7$ & $25$ & $75$ & $21$ & $101$ & $241$ \\
    dqn/cc\_cartpole & $43$ & $135$ & --- & $61$ & $121$ & --- & $259$ & $1185$ & --- & $421$ & $781$ & --- \\
    dqn/cc\_mountain\_car & $4$ & $61$ & --- & $21$ & $61$ & --- & $663$ & --- & --- & $1$ & --- & --- \\
    dqn/xland\_empty\_random & --- & --- & --- & --- & --- & --- & $727$ & --- & --- & --- & --- & --- \\
    dqn/xland\_four\_rooms & $97$ & $130$ & $217$ & $121$ & $241$ & $321$ & $19$ & $29$ & $33$ & $121$ & $201$ & $301$ \\
    ppo/brax\_fast & $48$ & $223$ & $389$ & $81$ & $201$ & $581$ & $75$ & $121$ & $889$ & $81$ & $241$ & $1081$ \\
    ppo/brax\_halfcheetah & $23$ & $81$ & --- & $41$ & $261$ & --- & $3$ & $121$ & --- & $41$ & --- & --- \\
    ppo/cc\_acrobot & $48$ & $63$ & $92$ & $81$ & $141$ & $181$ & $21$ & $37$ & $65$ & $61$ & $101$ & $141$ \\
    ppo/cc\_cartpole & $21$ & $63$ & $87$ & $21$ & $81$ & $161$ & $25$ & $57$ & $129$ & $21$ & $81$ & $161$ \\
    ppo/cc\_continuous\_mountain\_car & $43$ & $151$ & $270$ & $81$ & $161$ & $561$ & $11$ & $15$ & $33$ & $61$ & $101$ & $121$ \\
    ppo/cc\_mountain\_car & $898$ & --- & --- & $581$ & --- & --- & --- & --- & --- & --- & --- & --- \\
    ppo/cc\_pendulum & $21$ & $43$ & $210$ & $21$ & $141$ & $381$ & $25$ & $65$ & --- & $21$ & --- & --- \\
    ppo/xland\_door\_key & $1238$ & --- & --- & --- & --- & --- & $13$ & $23$ & $45$ & $761$ & $1241$ & --- \\
    ppo/xland\_empty\_random & $41$ & $145$ & $403$ & $61$ & $101$ & $281$ & $9$ & $19$ & $61$ & $81$ & $141$ & $201$ \\
    ppo/xland\_four\_rooms & $69$ & $123$ & $165$ & $101$ & $141$ & $221$ & $23$ & $27$ & $37$ & $101$ & $141$ & $221$ \\
    sac/brax\_fast & $21$ & $25$ & $120$ & $21$ & $41$ & $321$ & $37$ & $111$ & $201$ & $21$ & $141$ & $1001$ \\
    sac/brax\_halfcheetah & $23$ & $27$ & $427$ & $21$ & $61$ & $101$ & --- & --- & --- & $41$ & $81$ & $301$ \\
    sac/cc\_continuous\_mountain\_car & $63$ & $83$ & $125$ & $81$ & $101$ & $181$ & $15$ & $19$ & $29$ & $81$ & $141$ & $181$ \\
    sac/cc\_pendulum & $49$ & $85$ & $111$ & $121$ & $181$ & $321$ & $17$ & $25$ & $75$ & $121$ & $201$ & $301$ \\  
    \bottomrule
  \end{tabular}
  \end{adjustbox}
  \label{tab:threshold_crossing_mean}
\end{table}

\newpage

\subsection{Schedule $\beta_{\text{stop}}$}

It can be observed, e.g., in Figure~\ref{fig:rank_over_time}, that RAHBO (k=2) ranks best for small budgets, indicating that exploring with minimal repetition can be efficient at this stage. Inspired by this, we propose to increase the LCB of the incumbent. Recall \eqref{eq:inc_lcb} as 
\begin{equation}
B_t := \max_{\hyp \in \mathcal{D}_{t-1}} \mathrm{LCB}_t^{\mathrm{MV}}(\hyp).
\end{equation}
We set the confidence parameter $\beta$ in the $\mathrm{LCB}_t^{\mathrm{MV}}(\hyp)$ as $\beta_{\text{stop}, t}$ and increase it with an exponential schedule
\begin{equation}
    \label{app-eqn:exp_schedule}
    \beta_{\text{stop}, t} = \frac{e^{\gamma \tau} - 1}{e^{\gamma} - 1}
\end{equation}
or a sigmoid schedule
\begin{equation}
    \label{app-eqn:sig_schedule}
    \beta_{\text{stop}, t} = \frac{\sigma_{s}(\tau) - \sigma_{s}(0)}{\sigma_{s}(1) - \sigma_{s}(0)}
\end{equation} 
where $\tau = t/T$ is the normalized progress, with $t$ being the current evaluation and $T$ being the total evaluation budget, $\sigma_s(\tau) = 1/(1 + e^{s (\tau-0.5)})$ is a sigmoid function, $\gamma$ and $s$ are hyperparameter of the schedule. We choose different schedule hyperparameters and illustrate them in Figure~\ref{app-fig:schedule}. Moreover, we apply the same $\beta_{\text{stop}, t}$ to the UCB in~\eqref{eq:adaptive_rule}. We note that it is independent of the confidence parameter $\beta$ in the acquisition function. 

Figure~\ref{app-fig:rank_schedule} shows the domain-aggregated mean rank over time for ERAHBO with different schedules. We see that most schedules can outperform RAHBO $k=2$ after 200-500 evaluations, while some consistently outperform ERAHBO throughout the budget. Notably, on XLand, an exponential schedule with $\gamma=50$ can drastically improve ERAHOBO, as it not only outperforms a constant schedule but also ranks comparably close to RAHBO $k=2$. 
Even though the schedule introduces one additional hyperparameter, we see the performance is at least comparable with ERAHBO with a constant $\beta_{\text{stop}}$ as long as the initial $\beta_{\text{stop}, t}$ value of the schedule is small enough.

A small initial $\beta_{\text{stop}}$ benefits ERAHBO from a more optimistic triggering of the confidence-based resampling rule~\eqref{eq:adaptive_rule} early on; therefore, ERAHBO with such a schedule will start with investing fewer repetitions for rejection, however, still resamples thoroughly if the configuration is promising enough. As the $\beta_{\text{stop}}$ increases with the schedule, promising configurations are exploited with more repetitions. 

\label{app-sec:schedule_ERAHBO_beta}
\begin{figure}[t]
    \centering
    \includegraphics[width=\linewidth]{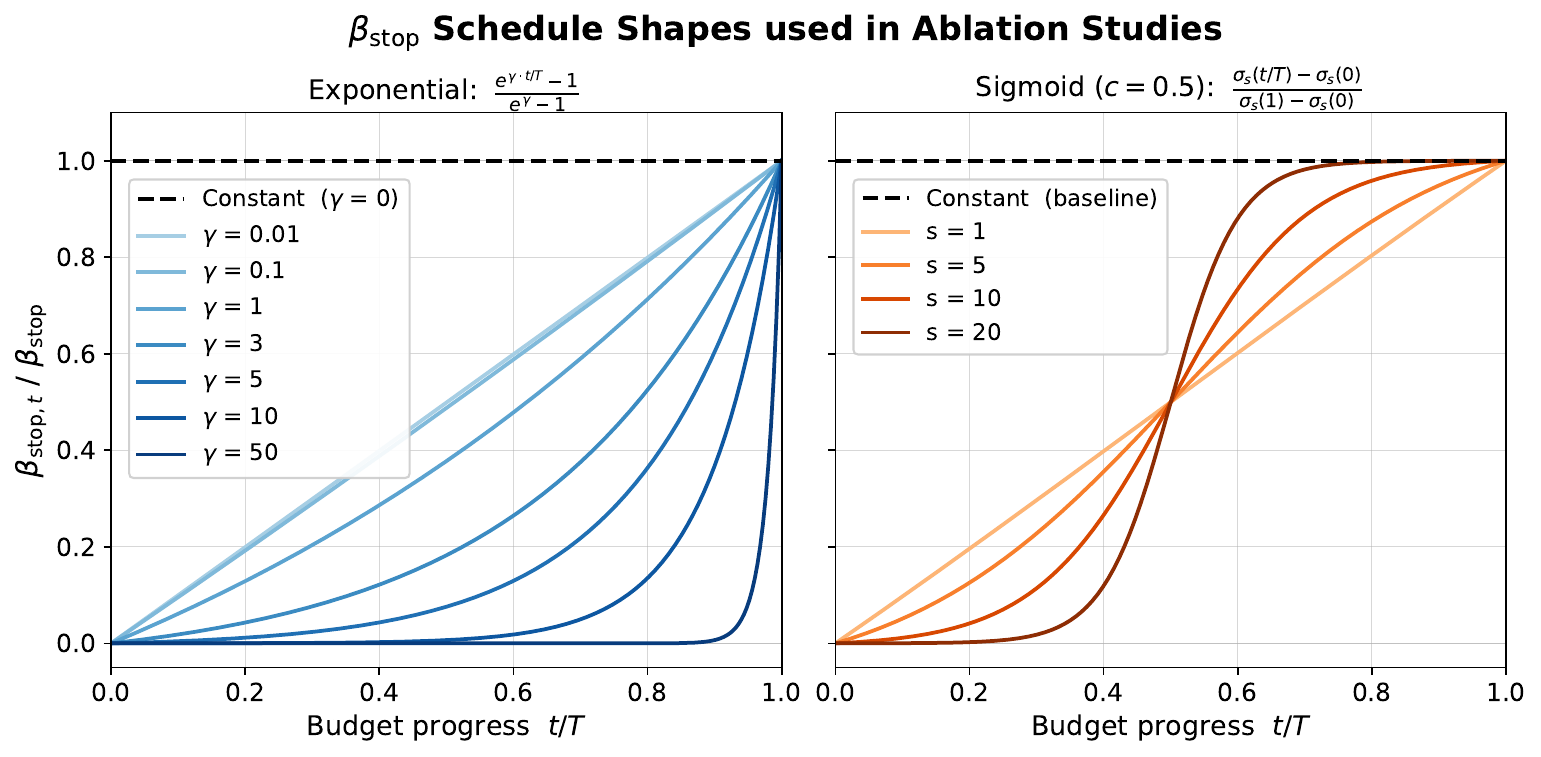}
    \caption{The $\beta_\text{stop}$ schedule used in the ablation study. The horizontal axis is the evaluation progress $t$ normalized by the total evaluation budget $T$, and the vertical axis is the instantaneous value $\beta_{\text{stop}, t}$, normalized by the maximum $\beta_{\text{stop}}$, at normalized progress $t/T$.}
    \label{app-fig:schedule}
\end{figure}

\begin{figure}[t]
    \centering
    \includegraphics[width=\linewidth]{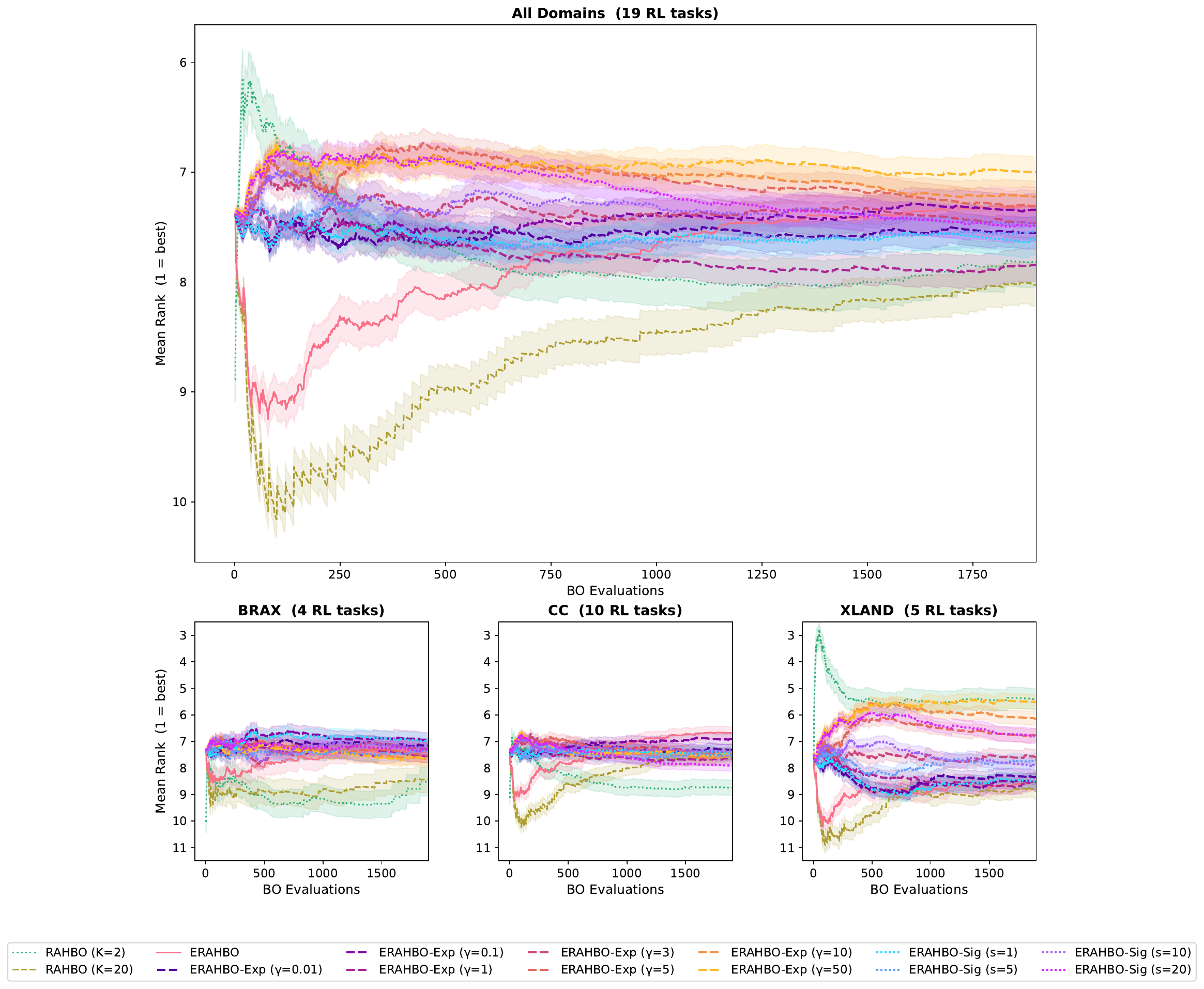}
    \caption{
    Domain aggregated mean rank over evaluation. The line and shaded area show the mean and $1\times$ standard error, respectively. Schedules with steep rises can outperform constant schedules. 
    }
    \label{app-fig:rank_schedule}
\end{figure}

\end{document}